%% file: paper.tex
\documentclass{article} 
\usepackage{style/iclr2023_conference,times}

\input{style/math_commands.tex}

\usepackage{algorithm}
\usepackage{algorithmic}

\usepackage{wrapfig}

\usepackage{xcolor}

\newcommand\blfootnote[1]{%
  \begingroup
  \renewcommand\thefootnote{}\footnote{#1}%
  \addtocounter{footnote}{-1}%
  \endgroup
}

\usepackage[]{svg}
\usepackage[hidelinks]{hyperref}
\usepackage{url}
\usepackage[scr=boondoxo,cal=dutchcal,bb=boondox,frak=euler]{mathalpha}

\usepackage{microtype}
\usepackage{graphicx}
\usepackage{subcaption}
\graphicspath{{figures/}}
\usepackage{booktabs} 
\usepackage{transparent} 

\usepackage{amsmath}
\usepackage{amssymb}
\usepackage{mathtools}
\usepackage{amsthm}

\usepackage[capitalize,noabbrev]{cleveref}

\usepackage{amsfonts}       
\usepackage{nicefrac}       
\usepackage{verbatim}
\usepackage{sidecap}
\usepackage{enumitem}
\usepackage[mathscr]{euscript}
\usepackage{blindtext}

\usepackage{multicol}
\usepackage{wrapfig}
\usepackage{pbox}

\newcommand{\eproof}{$\null\hfill\blacksquare$}
\newenvironment{proofsketch}{\par\noindent{\bf Proof Sketch:\ }}{\eproof}

\newtheorem{defn}{Definition}
\newtheorem{assm}{Assumption}
\newtheorem{lem}{Lemma}
\def\bbbr{{\rm I\!R}}
\newcommand{\yvec}{\mathbf{y}}
\newcommand\mysim{\stackrel{\mathclap{\tiny\mbox{iid}}}{\sim}}
\newenvironment{proofref}[1]{\par\noindent{\bf #1:\ }}{\eproof} 

\newcommand{\wvec}{\mathbf{w}}
\newcommand{\uvec}{\mathbf{u}}
\newcommand{\xvec}{\mathbf{x}}

\newcommand{\zerovec}{\mathbf{0}}

\newcommand{\pparams}{\mathbf{w}}
\newcommand{\propparams}{\mathbf{w}'}
\newcommand{\qparams}{\theta}

\newcommand{\Actions}{\mathcal{A}}
\newcommand{\States}{\mathcal{S}}

\newcommand{\defeq}{\doteq}
\newcommand{\Sigmamat}{\boldsymbol{\Sigma}}

\newcommand{\hatident}[1]{\hat{I}(#1)}
\newcommand{\thresh}[1]{f_{#1}}

\theoremstyle{plain}
\newtheorem{theorem}{Theorem}[section]

\theoremstyle{definition}

\theoremstyle{remark}

\title{Greedy Actor-Critic: A New Conditional Cross-Entropy Method for Policy Improvement}

\author{Samuel Neumann, Sungsu Lim, Ajin Joseph, Yangchen Pan, Adam White, Martha White \\ 
Department of Computing Science\\
University of Alberta\\
Edmonton, Alberta, Canada \\
\texttt{\string{sfneuman,amw8,whitem\string}@ualberta.ca} \\
}

\iclrfinalcopy 
\begin{document}

\maketitle

\begin{abstract}
Many policy gradient methods are variants of Actor-Critic (AC), where a value function (critic) is learned to facilitate
updating the parameterized policy (actor). The update to the actor involves a log-likelihood update weighted by the
action-values, with the addition of entropy regularization for soft variants. In this work, we explore an alternative
update for the actor, based on an extension of the cross entropy method (CEM) to condition on inputs (states). The idea
is to start with a broader policy and slowly concentrate around maximally valued actions, using a maximum likelihood update
towards actions in the top percentile per state. The speed of this concentration is controlled by a proposal policy,
that concentrates at a slower rate than the actor. We first provide a policy improvement result in an idealized setting,
and then prove that our conditional CEM (CCEM) strategy tracks a CEM update per state, even with changing action-values.
We empirically show that our GreedyAC algorithm, that uses CCEM for the actor update, performs better than Soft
Actor-Critic and is much less sensitive to entropy-regularization.
\blfootnote{Code available at \url{https://github.com/samuelfneumann/GreedyAC}.}
\end{abstract}

\vspace{-0.35cm}
\section{Introduction}
\vspace{-0.25cm}

Many policy optimization strategies update the policy towards the Boltzmann policy. This strategy became popularized by
Soft Q-Learning \citep{haarnoja2017reinforcement} and Soft Actor-Critic (SAC) \citep{haarnoja2018soft}, but has a long
history in reinforcement learning \citep{kober2008policy,neumann2011variational}. In fact, recent work
\citep{vieillard2020leverage,chan2021greedification} has highlighted that an even broader variety of policy optimization
methods can be seen as optimizing either a forward or reverse KL divergence to the Boltzmann policy, as in SAC. In fact,
even the original Actor-Critic (AC) update \citep{sutton1984temporal} can be seen as optimizing a reverse KL divergence,
with zero-entropy.

The use of the Boltzmann policy underlies many methods for good reason: it guarantees policy improvement
\citep{haarnoja2018soft}. More specifically, this is the case when learning entropy-regularized action-values
$Q_\tau^\pi$ for a policy $\pi$ with regularization parameter $\tau > 0$. The Boltzmann policy for a state is
proportional to $\exp(Q_\tau^\pi(s,a) \tau^{-1})$. The level of emphasis on high-valued actions is controlled by $\tau$:
the higher the magnitude of the entropy level (larger $\tau$),
the less the probabilities in the Boltzmann policy are peaked around
maximally valued actions.

This choice, however, has several limitations. The policy improvement guarantee is for the entropy-regularized MDP, rather than
the original MDP. Entropy regularization is used to encourage exploration \citep{ziebart2008maximum,mei2019principled}
and improve the optimization surface \citep{ahmed2018understanding,shani2019adaptive}, resulting in a trade-off between
improving the learning process and converging to the optimal policy.
Additionally, SAC and other methods are well-known to be sensitive to the entropy regularization
parameter \citep{pourchot2018cem}. Prior work has explored optimizing entropy during learning \citep{SACv2}, however,
this optimization introduces yet another hyperparameter to tune,
and this approach may be less performant than a simple grid
search (see Appendix~\ref{apdx:sac_ablation}).
It is reasonable to investigate alternative policy improvement approaches that could potentially
improve our actor-critic algorithms.

In this work we propose a new greedification strategy towards this goal.
The basic idea is to iteratively take the top percentile of actions, ranked according to the learned action-values. The
procedure slowly concentrates on the maximal action(s), across states, for the given action-values. The update itself is
simple: $N \in \mathbb{N}$
actions are sampled according to a \emph{proposal policy}, the actions are sorted based on the magnitude of the
action-values, and the policy is updated to increase the probability of the $\lceil \rho N \rceil$ maximally valued
actions for $\rho \in (0, 1)$.
We call this
algorithm for the actor Conditional CEM (CCEM), because it is an extension of the well-known Cross-Entropy Method (CEM)
\citep{rubinstein1999cem} to condition on inputs\footnote{CEM has been used for
policy optimization, but for two very different purposes. It has been used to directly optimize the policy
gradient objective \citep{mannor2003cem,szita2006tetris}. CEM has also been used to solve for the maximal action---running CEM each time we want to find $max_a' Q(S',a')$---for an algorithm called QT-Opt \citep{kalashnikov2018qt}. A follow-up algorithm adds an explicit deterministic policy to minimize a squared error to this maximal action \citep{simmons2019q} and another updates the actor with this action rather than the on-policy action \citep{shao2022grac}. We do not
directly use CEM, but rather extend the idea underlying CEM to provide a new policy update.}. We leverage theory for CEM
to validate that our algorithm concentrates on maximally valued actions across states over time. We introduce GreedyAC, a new AC
algorithm that uses CCEM for the actor.

GreedyAC has several advantages over using Boltzmann greedification. First, we show that our new greedification operator
ensures a policy improvement for the original MDP, rather than a different entropy-regularized MDP. Second, we can still
leverage entropy to prevent policy collapse, but only incorporate it into the proposal policy. This ensures the agent
considers potentially optimal actions for longer, but does not skew the actor. In fact, it is possible to decouple the
role of entropy for exploration and policy collapse within GreedyAC: the actor could have a small amount of entropy to
encourage exploration, and the proposal policy a higher level of entropy to avoid policy collapse. Potentially because
of this decoupling, we find that GreedyAC is much less sensitive to the choice of entropy regularizer, as compared to
SAC.
This design of the algorithm may help it avoid getting stuck in a locally optimal action, and empirical evidence for CEM
suggests it can be quite effective for this purpose \citep{rubinstein2004cross}. In addition to our theoretical support
for CCEM, we provide an empirical investigation comparing GreedyAC, SAC, and a vanilla AC, highlighting that GreedyAC
performs consistently well, even in problems like the Mujoco environment Swimmer and pixel-based control where SAC
performs poorly.


\vspace{-0.25cm}
\section{Background and Problem Formulation}%
\label{sec:background_and_problem_formulation}
\vspace{-0.25cm}
The interaction between the agent and environment is formalized by a Markov decision process
$\left(\mathscr{S}, \mathscr{A}, \mathscr{P}, \mathscr{R}, \gamma \right)$, where $\mathscr{S}$ is the state space,
$\mathscr{A}$ is the action space, $\mathscr{P}: \mathscr{S} \times \mathscr{A} \times \mathscr{S} \to \left[ 0, \infty
\right)$ is the one-step state transition dynamics, $\mathscr{R}: \mathscr{S} \times \mathscr{A} \times \mathscr{S} \to
\mathbb{R}$ is the reward function, and $\gamma \in [0, 1]$ is the discount rate. We assume an episodic problem setting,
where the start state $S_0 \sim d_0$ for start state distribution $d_0: \mathscr{S} \rightarrow [0, \infty)$ and the
length of the episode $T$ is random, depending on when the agent reaches termination. At each discrete timestep $t = 1,
2, \ldots, T$, the agent finds itself in some state $S_{t}$ and selects an action $A_{t}$ drawn from its stochastic
policy $\pi: \mathscr{S} \times \mathscr{A} \to [0, \infty)$. The agent then transitions to state $S_{t+1}$ according to $\mathscr{P}$ and observes a scalar reward $R_{t+1} \doteq \mathscr{R}(S_{t}, A_{t}, S_{t+1})$.

For a parameterized policy $\pi_\pparams$ with parameters $\pparams$, the agent attempts to maximize the objective
$J(\pparams) =\mathbb{E}_{\pi_\pparams} [ \sum_{t=0}^{T} \gamma^{t} R_{t+1} ]$, where the expectation is according to start
state distribution $d_0$, transition dynamics $\mathscr{P}$, and policy $\pi_\pparams$.  Policy gradient methods, like
REINFORCE \citep{williams1992simple}, attempt to obtain (unbiased) estimates of the gradient of this objective
to directly update the policy.

The difficulty is that the policy gradient is expensive to sample, because it requires sampling return trajectories from
states sampled from the visitation distribution under $\pi_\pparams$, as per the policy gradient theorem
\citep{sutton2000policy}. Theory papers analyze such an idealized algorithm
\citep{kakade2002approximately,agarwal2021theory}, but in practice this strategy is rarely used.
Instead, it is much more common to (a) ignore bias in the state distribution \citep{thomas2014bias,imani2018off,nota2019policy} and
(b) use biased estimates of the return, in the form of a value function critic.  The
action-value function $Q^\pi(s,a) \defeq \mathbb{E}_\pi[\sum_{k=1}^{T-t} \gamma^{t} R_{t+k} | S_t = s, A_t = a]$
is the expected return from a given state and action, when following policy $\pi$. Many PG methods---specifically
variants of Actor-Critic---estimate these action-values with parameterized $Q_\qparams(s, a)$, to use the update $Q_\qparams(s, a) \nabla
\ln \pi_\pparams(a | s)$ or one with a baseline $[Q_\qparams(s, a) - V(s)]\nabla \ln \pi_\pparams(a | s)$ where the
value function $V(s)$ is also typically learned. The state $s$ is sampled from a replay buffer, and $a \sim
\pi_\pparams(\cdot | s)$, for the update.

There has been a flurry of work, and success, pursuing this path, including methods such as OffPAC
\citep{degris2012offpac}, SAC \citep{haarnoja2018soft}, SQL \citep{haarnoja2017reinforcement}, TRPO
\citep{schulman2015trust} and many other variants of related ideas
\citep{peters2010relative,silver2014deterministic,schulman2015high,lillicrap2015continuous,wang2016sample,gu2016qprop,schulman2017proximal,abdolmaleki2018maximum,mei2019principled,vieillard2019deep}.
Following close behind are unification results that make sense of this flurry of work
\citep{tomar2020mirror,vieillard2020leverage,chan2021greedification,lazic2021optimization}. They highlight that many
methods include a mirror descent component---to minimize KL to the most recent policy---and an entropy-regularization
component \citep{vieillard2020leverage}. In particular, these methods are better thought of as (approximate) policy
iteration approaches that update towards the Boltzmann policy, in some cases using a mirror descent update.
The Boltzmann policy $\mathcal{B}_\tau Q(s,a)$ for a given $Q$ is
\begin{equation} \label{eqn:Boltzmann}
\mathcal{B}_\tau Q(s,a) = \frac{\exp(Q(s,a) \tau^{-1})}{\int_\Actions \exp(Q(s,b) \tau^{-1}) db}
\end{equation}
for
entropy parameter $\tau$. As $\tau \rightarrow 0$, this policy puts all weight on greedy actions. As $\tau \rightarrow
\infty$, all actions are weighted uniformly. This policy could be directly used as the new greedy policy.
However, because it is expensive to sample from
$\mathcal{B}_\tau Q(s,a)$, typically a parameterized policy $\pi_\pparams$ is learned to approximate $\mathcal{B}_\tau Q(s,a)$, by
minimizing a KL divergence.
As the entropy goes to zero, we get an unregularized update that corresponds to the
vanilla AC update \citep{chan2021greedification}.

\vspace{-0.25cm}
\section{Conditional CEM}%
\label{sec:ccem_algorithm}
\vspace{-0.25cm}

Though using the Boltzmann policy has been successful, it does have some limitations. The primary limitation is that it
is sensitive to the choice of entropy \citep{pourchot2018cem,chan2021greedification}. A natural question is what other strategies we can
use for this greedification step in these approximate policy iteration algorithms, and how they compare to this common
approach.  We propose and motivate a new approach in this section, and then focus the paper on providing insight into
its benefits and limitations, in contrast to using the Boltzmann policy.

Let us motivate our approach, by describing the well-known global optimization algorithm called the Cross Entropy Method
(CEM) \citep{rubinstein1999cem}. Global optimization strategies are designed to find the global optimum of a general
function $f(\beta)$ for some parameters $\beta$. For example, for parameters $\beta$ of a neural network, $f$ may be
the loss function on a sample of data. An advantage of these methods is that they do not rely on gradient-based
strategies, which are prone to getting stuck in local optima. Instead, they use randomized search strategies, that have
optimality guarantees in some settings \citep{hu2012stochastic} and have been shown to be effective in practice
\citep{peters2007reinforcement,hansen2003cma,szita2006tetris,salimans2017evolution}.

CEM maintains a distribution $p(\beta)$ over parameters $\beta$, iteratively narrowing the range of plausible
solutions. The algorithm maintain a current threshold $\thresh{t}$, that slowly increases over time as it narrows on the
maximal $\beta$. On iteration $t$, $N$ parameter vectors $\beta_1, \ldots, \beta_N$ are sample from $p_t$; the
algorithm only keeps $\beta_1^*, \ldots, \beta_h^*$ where $f(\beta^*_i) \ge \thresh{t}$ and discards the rest.  The
KL divergence is reduced between $p_t$ and this empirical distribution $\hat{I} = \{\beta_1^*, \ldots, \beta_h^*\}$,
for $h \le N$. This step corresponds to increasing the likelihood of the $\beta$ in the set $\hat{I}$. Iteratively, the
distribution over parameters $p_t$ narrows around $\beta$ with higher values under $f$.  To make it more likely to find the global optimum, the initial distribution $p_0$ is a wide distribution, such as a Gaussian distribution with mean zero $\mu_0 =
\zerovec$ and a diagonal covariance $\Sigmamat_0$ of large magnitude.

\begin{wrapfigure}{}{0.4\textwidth}
\vspace{-0.75cm}
\begin{minipage}{\linewidth}
\begin{algorithm}[H]
\caption{Percentile Empirical Distribution($N, \rho$)}
\begin{algorithmic}\label{alg_ccem1}
	\STATE Evaluate and sort in descending order: \\
	$Q_\theta(S_t, a_{i_1}) \geq \ldots \geq Q_\theta(S_t,a_{i_N})$
	 \RETURN $\hatident{S_t} = \{a_{i_1}, \ldots, a_{i_h}\}$ \\
	 (where $h=\lceil \rho N \rceil$ )
\end{algorithmic}
\end{algorithm}
\end{minipage}
\vspace{-0.5cm}
\end{wrapfigure}

CEM attempts to find the single-best set of optimal parameters for a single optimization problem. The straightforward
use in reinforcement learning is to learn the single-best set of policy parameters $\pparams$ \citep{szita2006tetris,
mannor2003cem}.  Our goal, however, is to (repeatedly) find maximally valued actions $a^*$ conditioned on
each state for $Q(s, \cdot)$. The global optimization
strategy could be run on each step to find the exact best action for each current state, as in QT-Opt
\citep{kalashnikov2018qt} and follows-ups \citep{simmons2019q,shao2022grac}, but this is expensive and throws away prior information about the function surface obtained
when previous optimizations were executed.

We extend CEM to be (a) conditioned on state and (b) learned iteratively over time.  The key modification when extending
CEM to Conditional CEM (CCEM), to handle these two key differences, is to introduce another \emph{proposal policy} that
concentrates more slowly. This proposal policy is entropy-regularized to ensure that we keep a broader set of potential
actions when sampling, in case changing action-values are very different since the previous update to that state. The
main policy (the actor) does not use entropy regularization, allowing it to more quickly start acting according to
currently greedy actions, without collapsing. We visualize this in Figure \ref{fig_ccem}.

\begin{figure*}[ht]
	 \vspace{-0.75cm}
	\centering
 	\begin{subfigure}[b]{0.72\textwidth}
		\fontsize{6pt}{6pt}\selectfont
		\includesvg[width=\linewidth]{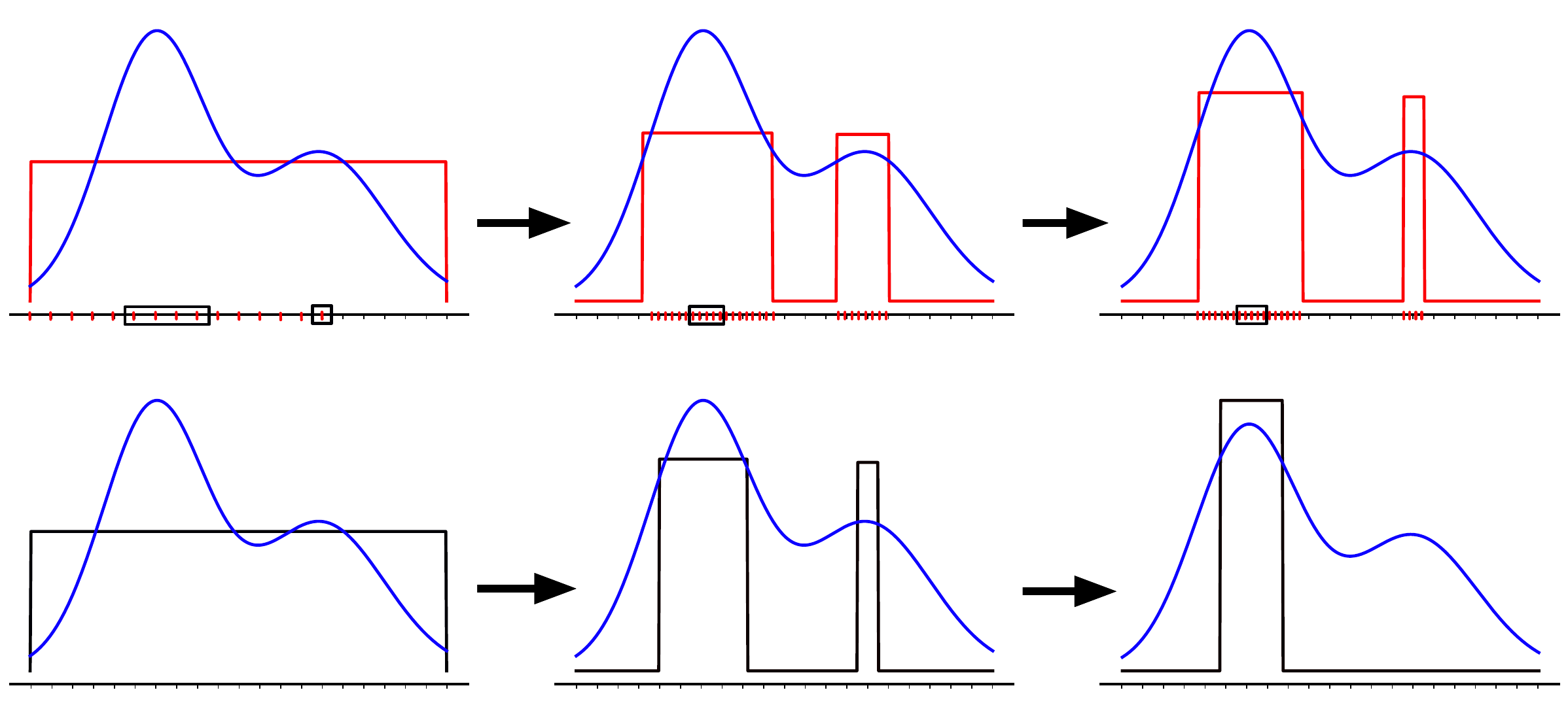}
	\end{subfigure}
\begin{subfigure}[b]{0.27\textwidth}
		\fontsize{6pt}{6pt}\selectfont
		\includesvg[width=\linewidth]{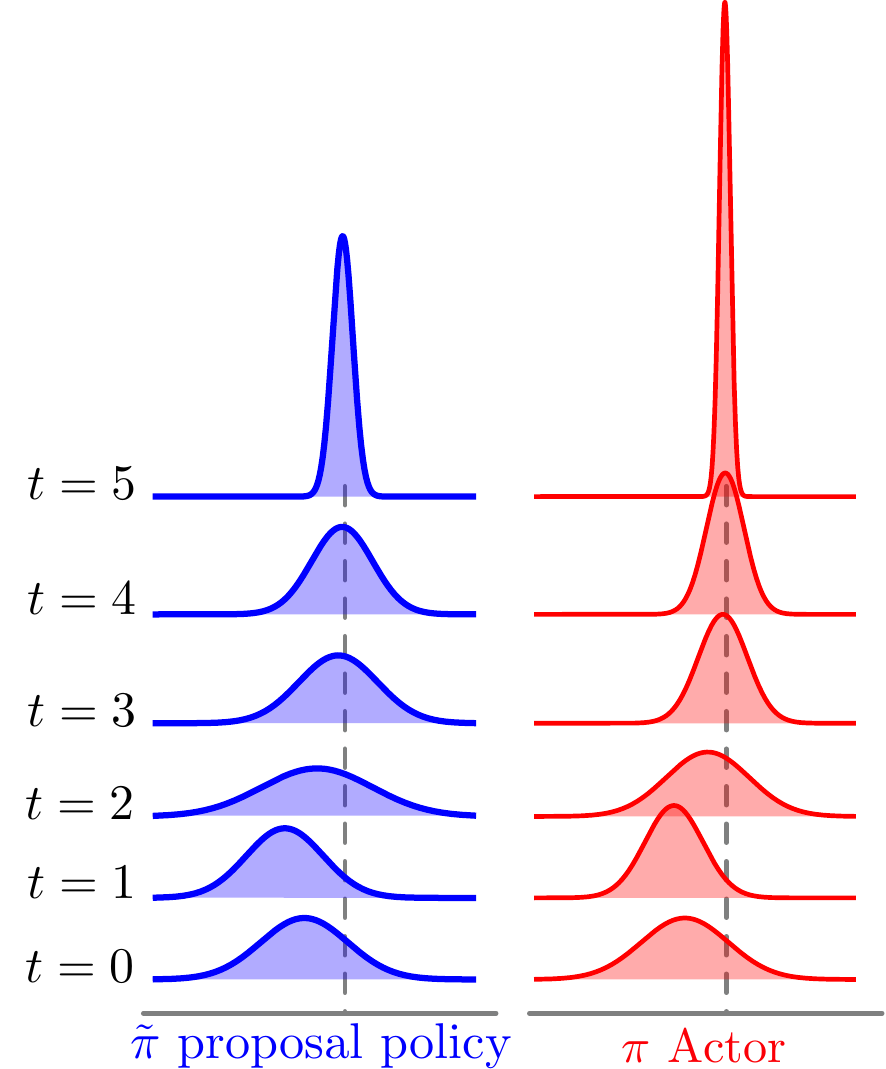}
	\end{subfigure}
	\vspace{-0.25cm}
  \caption{In the {\bf left} figure we see multiple updates for both policies of the CCEM in a single state.
	  We use uniform policies, for interpretability.
	  In the {\bf rightmost} figure, we show an actual
	  progression of CCEM with Gaussian policies, when executed on the action-values depicted in the leftmost figure.
	  The Actor policy (in black) concentrates more quickly than the Proposal policy (in red).}  \label{fig_ccem}
	  \vspace{-0.25cm}
\end{figure*}

The CCEM algorithm is presented in Algorithm \ref{alg_ae_ccem}.
On each step, the proposal policy, $\tilde \pi_{\pparams'_t}(\cdot | S_t)$, is
sampled to provide a set of actions $a_1, \ldots, a_N$ from which we construct the empirical distribution
$\hatident{S_t} = \{a^*_1, \ldots, a_h^*\}$ of maximally valued actions.  The actor parameters $\pparams_t$ are updated using a gradient ascent
step on the log-likelihood of the actions $\hatident{S_t}$. The proposal parameters $\pparams'_t$ are updated using a similar update, but with
an entropy regularizer.
To obtain $\hatident{S_t}$,
we select $a_i^* \subset \{a_1, \ldots,
a_N\}$ where $Q(S_t, a_i^*)$ are in the top $(1-\rho)$ quantile values.  For example, for $\rho = 0.2$, approximately
the top 20\% of actions are chosen, with $h = \lceil \rho N \rceil$. Implicitly, $\thresh{t}$ is $Q_\theta(S_t, a^*_h)$
for $a^*_h$ the action with the lowest value in this top percentile.  This procedure is summarized in Algorithm
\ref{alg_ccem1}.

Greedy Actor-Critic, in Algorithm \ref{alg_greedy-ac}, puts this all together. We use experience replay, and
the CCEM algorithm on a mini-batch. The updates involve obtaining the sets $\hatident{S}$ for every $S$
in the mini-batch $B$ and updating with the gradient $\tfrac{1}{|B|}\sum_{S \in B} \sum_{a \in \hatident{S}}
\nabla_\pparams \ln \pi_\pparams(a|S)$.  The Sarsa update to the critic involves (1) sampling an on-policy action from
the actor $A' \sim
\pi_{\pparams}(\cdot | S')$ for each tuple in the mini-batch and (2) using the update
$\tfrac{1}{|B|}\sum_{(S,A,S',R,A') \in B} (R + \gamma Q_\qparams(S', A') - Q_\qparams(S, A))\nabla_\qparams
Q_\qparams(S, A)$. Other critic updates are possible; we discuss alternatives and connections to related algorithms in Appendix \ref{app_connections}.

\vspace{-0.4cm}
\begin{minipage}{.45\textwidth}
\begin{algorithm}[H]
\caption{Conditional CEM for the Actor}\label{alg_ae_ccem}
\begin{algorithmic}
	\STATE \textbf{Input: } $S_t$ and $Q_\qparams$, $N \in \mathbb{N}$, $\rho \in (0, 1)$
	\IF{actions discrete and $| \Actions | \le 1/\rho$}
	\STATE $\hatident{S_t} = \argmax_{a \in \Actions} Q_\qparams(S_t, a)$
	\ELSE
	\STATE Sample $N$ actions $a_i \sim \tilde \pi_{\propparams}(\cdot | S_{t})$
	\STATE Obtain $\hatident{S_t}$ using Algorithm \ref{alg_ccem1}
	\ENDIF
	\STATE $\pparams \gets \pparams + \alpha_{p,t} \sum_{a \in \hatident{S_t}} \nabla_\pparams \ln \pi_\pparams(a|S_{t})$
	\STATE $\!\!\propparams \!\gets\! \propparams \!+\! \alpha_{p,t} [\sum_{a \in \hatident{S_t}} \!\!\nabla_{\propparams}
	\ln \tilde \pi_{\propparams}\!(a|S_{t}) \!+\! \tau \nabla_{\propparams} \mathcal{H}(\tilde \pi_{\propparams}\!(\cdot|S_{t}))]$
\end{algorithmic}
\end{algorithm}
\end{minipage}%
\begin{minipage}{.05\textwidth}
\hspace{0.05cm}
\end{minipage}
\begin{minipage}{.5\textwidth}
\begin{algorithm}[H]
\caption{Greedy Actor-Critic}\label{alg_greedy-ac}
\begin{algorithmic}
	\STATE Initialize parameters $\qparams, \pparams, \pparams'$, replay buffer $\mathcal{B}$
	\STATE Obtain initial state $S$
	\WHILE{agent interacting with the environment}
	\STATE Take action $A \sim \pi_{\pparams}(\cdot | S)$, observe $R$, $S'$
	\STATE Add $(S, A, S', R) $ to the buffer $\mathcal{B}$
	\STATE Grab a random mini-batch $B$ from buffer $\mathcal{B}$
	\STATE Update $\qparams$ using Sarsa for policy $\pi_\pparams$ on $B$
	\STATE Update $\pparams, \pparams'$ using Algorithm \ref{alg_ae_ccem} on $B$.
	\ENDWHILE
\end{algorithmic}
\end{algorithm}
\end{minipage}

\textbf{CCEM for Discrete Actions.} Although we can use the same algorithm for discrete actions, we can make it simpler when we have a small number of discrete actions. Our algorithm is designed around the fact that it is difficult to solve for the maximal action for $Q_\qparams(S_t, a)$ for continuous actions; we slowly identify this maximal action across states. For a small set of discrete actions, it is easy to get this maximizing action. If $|\Actions| < 1/\rho$, then the top percentile consists of the one top action (or the top actions if there are ties); we can directly set $\hatident{S_t} = \argmax_{a \in \Actions} Q_\qparams(S_t, a)$ and do not need to maintain a proposal policy. For this reason, we focus our theory on the continuous-action setting, which is the main motivation for using CEM for the actor update.


\vspace{-0.25cm}
\section{Theoretical Guarantees}
\vspace{-0.25cm}

In this section, we motivate that the target policy underlying CCEM guarantees policy improvement, and characterize the
ODE underlying CCEM. We show it tracks a CEM update in expectation across states and slowly concentrates around
maximally valued actions even while the action-values are changing.

\newcommand{\threshold}{\text{thresh}}

\vspace{-0.25cm}
\subsection{Policy Improvement under an Idealized Setting}
\vspace{-0.25cm}

We first consider the setting where we have access to $Q^\pi$, as is typically done for characterizing the policy
improvement properties of an operator \citep{haarnoja2018soft,ghosh2020operator,chan2021greedification} as well as for
the original policy improvement theorem \citep{sutton2018reinforcement}. Our update moves our policy towards a
\emph{percentile-greedy} policy that redistributes probability solely to the $(1-\rho)$-quantile according to magnitudes
under $Q(s, a)$. More formally, let $f_{Q}^{\rho}(\pi; s)$ be the threshold such that $\int_{\{ a \in \Actions | Q(s, a)
\ge f_{Q}^{\rho}(\pi; s) \}} \pi(a | s) da = \rho$, namely that gives the set of actions in the top $1-\rho$ quantile,
according to magnitudes under $Q(s, \cdot)$. Then we can define the percentile-greedy policy as
\begin{equation}\label{eq_percentile}
	\pi_\rho(a | s, Q, \pi) = \begin{cases} \pi(a | s)/\rho & Q(s, a) \ge \threshold f_{Q}^{\rho}(\pi; s) \\
	0 & \text{else} \end{cases}
\end{equation}
	where diving by $\rho$ renormalizes the distribution.
	Computing this policy would be onerous; instead, we only sample the KL divergence to this policy, using a sample
	percentile. Nonetheless, this percentile-greedy policy represents the target policy that the actor updates towards
	(in the limit of samples $N$ for the percentile).

	Intuitively, this target policy should give policy improvement, as it redistributes weight for low valued actions
	proportionally to high-valued actions.  We formalize this in the following theorem. We write $\pi_\rho(a | s)$
	instead of $\pi_\rho(a | s, Q^\pi, \pi)$, when it is clear from context.

\begin{theorem}
For a given policy $\pi$, action-value $Q^\pi$ and $\rho > 0$, the percentile-greedy policy $\pi_\rho$ in $\pi$ and $Q^\pi$ is guaranteed to be at least as good as $\pi$ in all states:
\begin{align*}
\int_\Actions \pi_\rho(a | s, Q^\pi, \pi) Q^{\pi_\rho}(s, a) da \ge \int_\Actions \pi(a | s) Q^\pi(s,a) da
\end{align*}
\end{theorem}
\begin{proof}
The proof is a straightforward modification of the policy improvement theorem. Notice that
\begin{align*}
\int_\Actions \pi_\rho(a | s) Q^\pi(s,a) da \qquad &=
\int\limits_{\{ a \in \Actions | Q(s, a) \ge f_{Q}^{\rho}(\pi; s) \}} \!\!\!\!\!\!\!\!\!\!\!\!\!\!\!\!\!\!\!\!
\frac{\pi(a | s)}{\rho} Q^\pi(s,a) da
&\ge \quad \int\limits_{\Actions} \pi(a | s) Q^\pi(s,a) da
\end{align*}
by the definition of percentiles, for any state $s$. Rewriting $\int_\Actions \pi(a | s) Q^\pi(s,a)da = \mathbb{E}_\pi[Q^\pi(s, A)]$,
\begin{align*}
V^\pi(s) &= \mathbb{E}_\pi[Q^\pi(s, A)] \le \mathbb{E}_{\pi_\rho}[Q^\pi(s, A)]
= \mathbb{E}_{\pi_\rho}[R_{t+1} + \gamma \mathbb{E}_\pi[Q^\pi(S_{t+1}, A_{t+1}) | S_t = s]\\
&\le \mathbb{E}_{\pi_\rho}[R_{t+1} + \gamma \mathbb{E}_{\pi_\rho}[Q^\pi(S_{t+1}, A_{t+1})] | S_t = s]\\
&\le \mathbb{E}_{\pi_\rho}[R_{t+1} + \gamma R_{t+2} + \gamma^2 \mathbb{E}_{\pi}[Q^\pi(S_{t+2}, A_{t+2})] | S_t = s]\\
&\ldots\\
&\le \mathbb{E}_{\pi_\rho}[R_{t+1} + \gamma R_{t+2} + \gamma^2 R_{t+3} + \ldots \gamma^{T-1} R_T | S_t = s]
= \mathbb{E}_{\pi_\rho}[Q^{\pi_\rho}(s, A)] = V^{\pi_\rho}(s)
\end{align*}
\par\vspace{-0.9cm}
\end{proof}
\vspace{-0.25cm}
This result is a sanity check to ensure the target policy is sensible in our update.
Note that the Boltzmann policy only guarantees improvement under the entropy-regularized action-values.

\vspace{-0.25cm}
\subsection{CCEM Tracks the Greedy Action}\label{sec_ccem_theory}
\vspace{-0.25cm}

Beyond the idealized setting, we would like to understand the properties of the stochastic algorithm.  CCEM is not a
gradient descent approach, so we need to reason about its dynamics---namely the underlying ODE. We expect CCEM to behave like CEM per state, but with some qualifiers. First, CCEM uses a parameterized policy conditioned
on state, meaning that there is aliasing between the action distributions per state. CEM, on the other hand, does not
account for such aliasing. We identify conditions on the parameterized policy and use an ODE that takes expectations
over states.

Second, the function we are attempting to maximize is also changing with time, because the action-values are updating.
We address this issue using a two-timescale stochastic approximation approach, where the action-values $Q_{\theta}$
change more slowly than the policy, allowing the policy to track the maximally valued actions. The policy itself has two
timescales, to account for its own parameters changing at different timescales. Actions for the maximum likelihood step
are selected according to older (slower) parameters $\pparams^{\prime}$, so that it is as if the primary (faster)
parameters $\pparams$ are updated using samples from a fixed distribution. These two policies correspond to our proposal
policy (slow) and actor (fast).

We show that the ODE for the CCEM parameters $\pparams_t$ is based on the gradient
\vspace{-0.10cm}
\begin{equation*}
\nabla_{\pparams(t)}\mathbb{E}_{\substack{S \sim \nu, A \sim \pi_{\pparams^{\prime}}(\cdot \vert
S)}}\Big[I_{\{Q_{\theta}(S, A) \geq f_{\theta}^{\rho}(\pparams^{\prime}; S)\}} \ln \pi_{\pparams(t)}(A|S)\Big]
\vspace{-0.10cm}
\end{equation*}
where $\theta$ and $\pparams^{\prime}$ are changing at slower timescales, and so effectively fixed from
the perspective of the faster changing $\pparams_t$.  The term per-state is exactly the update underlying CEM, and so we
can think of this ODE as one for an expected CEM Optimizer, across states for parameterized policies. We say that CCEM
\emph{tracks} this expected CEM Optimizer, because $\theta$ and $\pparams^{\prime}$ are changing with time.

We provide an informal theorem statement here for Theorem \ref{thm_main}, with a proof-sketch. The main result,
including all conditions, is given in Appendix \ref{app_theory_main}. We discuss some of the (limitations of the)
conditions after the proof sketch.

\textbf{Informal Result:} Let $\theta_t$ be the action-value parameters with stepsize $\alpha_{q,t}$, and $\wvec_t$ be
the policy parameters with stepsize $\alpha_{a,t}$, with $\wvec'_t$ a more slowly changing set of policy parameters set
to $\wvec'_t = (1-\alpha'_{a,t}) \wvec'_t + \alpha'_{a,t} \wvec_t$ for stepsize $ \alpha'_{a,t} \in (0,1]$.  Assume: (1) States $S_t$ are sampled from a fixed marginal
	distribution. (2) $\nabla_{\wvec}\ln{\pi_{\pparams}(\cdot \vert s)}$ is locally Lipschitz w.r.t. $\pparams$,
	$\forall s \in \States$. (3) Parameters $\wvec_t$ and $\theta_t$ remain bounded almost surely.  (4) Stepsizes
	are chosen for three different timescales: $\wvec_t$ evolves faster than $\wvec'_t$ and $\wvec'_t$ evolves faster
than $\theta_t$.
%
Under these four conditions,
the CCEM Actor tracks the expected CEM Optimizer.  \begin{proofsketch} The stochastic
	update to the Actor is not a direct gradient-descent update. Each update to the Actor is a CEM update, which
	requires a different analysis to ensure that the stochastic noise remains bounded and is asymptotically negligible.
	Further, the classical results of CEM also do not immediately apply, because such updates assume distribution
	parameters can be directly computed. Here, distribution parameters are conditioned on state, as outputs from a
	parametrized function. We identify conditions on the parametrized policy to ensure well-behaved CEM updates.

	The multi-timescale analysis allows us to focus on the updates of the Actor $\wvec_t$, assuming the action-value
	parameter $\theta$ and action-sampling parameter $\wvec^{\prime}$ are quasi-static. These parameters are allowed to
	change with time---as they will in practice---but are moving at a sufficiently slower timescale relative to
	$\wvec_t$ and hence the analysis can be undertaken as if they are static.

	The first step in the proof is to formulate the update to the weights as a projected stochastic recursion---simply
	meaning a stochastic update where after each update the weights are projected to a compact, convex set to keep them
	bounded. The stochastic recursion is reformulated into a summation involving the mean vector field
	$g^\theta(\wvec_t)$ (which depends on the action-value parameters $\theta$), martingale noise, and a loss term
	$\ell^\theta_t$ that is due to having approximate quantiles. The key steps are then to show almost surely that the
	mean vector field $g^\theta$ is locally Lipschitz, the martingale noise is quadratically bounded and that the loss
	term $\ell^\theta_t$ decays to zero asymptotically. For the first and second, we identify conditions on the policy
parameterization that guarantee these. For the final case, we adapt the proof for sampled quantiles approaching true
quantiles for CEM, with modifications to account for expectations over the conditioning variable, the state.
\end{proofsketch}

This result has several limitations. First, it does not perfectly characterize the CCEM algorithm that we actually use.
We do not use the update $\wvec'_t = (1-\alpha'_{a,t}) \wvec'_t + \alpha'_{a,t}\wvec_t$, and instead use entropy
regularization to make $\wvec'_t$ concentrate more slowly than $\wvec_t$. The principle is similar; empirically we
found entropy regularization to be an effective strategy to achieve this condition.

Second, the theory assumes the state distribution is fixed, and not influenced by $\pi_{\pparams}$. It is standard to
analyze the properties of (off-policy) algorithms for fixed datasets as a first step, as was done for Q-learning
\citep{jaakkola1994convergence}. It allows us to separate concerns, and just ask: does our method concentrate on maximal
actions across states?
An important next step is to characterize the full Greedy
Actor-Critic algorithm, beyond just understanding the properties of the CCEM component.

\vspace{-0.25cm}
\section{Empirical Results}
\vspace{-0.25cm}

We are primarily interested in investigating sensitivity to hyperparameters. This sensitivity reflects
how difficult it can be to get AC methods working on a new task---relevant for both applied settings and research.
AC methods have been notoriously difficult to tune due to the interacting time scales of the actor and
critic \citep{degris2012model}, further compounded by the sensitivity in the entropy scale.
The use of modern optimizers may reduce some of the sensitivity in stepsize selection;
these experiments help understand if that is the
case. Further, a very well-tuned algorithm may not be representative of performance across problems. We particularly
examine the impacts of selecting a single set of hyperparameters across environments, in contrast to tuning per
environment.

We chose to conduct experiments in small, challenging domains appropriately sized for
extensive experiment repetition. Ensuring significance in results and carefully exploring hyperparameter
sensitivity required many experiments. Our final plots required $\sim$30,000 runs across
all environments, algorithms, and hyperparameters.
Further,
contrary to popular belief, classic control domains are a challenge for Deep RL agents \citep{ghiassian2020improving},
and performance differences in these environments have been shown to extend to larger environments
\citep{obando2021revisiting}.

\vspace{-0.25cm}
\subsection{Algorithms}
\vspace{-0.25cm}

We focus on comparing GreedyAC to Soft Actor-Critic (SAC) both since this allows us to compare to a method that uses
the Boltzmann target policy on action-values and because SAC continues to have the most widely reported
success\footnote{See \url{https://spinningup.openai.com/en/latest/spinningup/bench.html}}. We additionally include
VanillaAC as a baseline, a basic AC variant which does not include any of the tricks SAC utilizes
to improve performance, such as action reparameterization to estimate the policy gradient
or double Q functions to mitigate maximization bias.
For discrete actions, policies are parameterized using Softmax distributions. For continuous
actions, policies are parameterized using Gaussian distributions, except SAC which uses a squashed Gaussian policy as
per the original work. We tested SAC with a Gaussian policy, and it performed worse.
All algorithms use neural networks. Feedforward networks consist of two hidden layers of 64 units (classic control
environments) or 256 units (Swimmer-v3 environment). Convolutional layers
consists of one convolutional layer with 3 kernels of size 16 followed by a fully connected layer of size 128.
All algorithms use the Adam optimizer \citep{kingma2014adam}, experience replay, and target networks for the value
functions. See Appendix~\ref{apdx:hypers} for a full discussion of hyperparameters.

\vspace{-0.25cm}
\subsection{Environments}
\vspace{-0.25cm}

We use the classic versions of Mountain Car \citep{sutton2018reinforcement}, Pendulum \citep{degris2012model},
and Acrobot \citep{sutton2018reinforcement}.
Each environment is run with
both continuous and discrete action spaces; states are continuous.
Discrete actions consist of the two extreme continuous actions and 0.
All environments use a discount factor of $\gamma = 0.99$,
and episodes are cut off at 1,000 timesteps, teleporting the agent back to the start state (but not causing
termination).
To demonstrate the potential of GreedyAC at scale, we also include experiments on Freeway and Breakout
from MinAtar \citep{young19minatar} as well as on Swimmer-v3 from OpenAI Gym
\citep{aigym}. On MinAtar, episodes are cutoff at 2,500 timesteps.

In Mountain Car, the goal is to drive an underpowered car up a hill. State consists of the position in $\left[-1.2,
0.6\right]$ and velocity in $\left[-0.7, 0.7 \right]$. The agent starts in a random position in $\left[ -0.6, -0.4
\right]$ and velocity 0.  The action is the force to apply to the car, in $[-1, 1]$. The reward is -1 per step.

In Pendulum, the goal is to hold a pendulum with a fixed base in a vertical position. State consists of the angle
(normalized in $\left[-\pi, \pi\right)$) and angular momentum (in $\left[-1, 1\right]$).  The agent starts with the
pendulum facing downwards and 0 velocity. The action is the torque applied to the fixed base, in $\left[-2, 2\right]$.
The reward is the cosine of the angle of the pendulum from the positive y-axis.

In Acrobot, the agent controls a doubly-linked pendulum with a fixed base. The goal is to swing the second link one
link's length above the fixed base. State consists of the angle of each link (in $\left[-\pi, \pi\right)$) and the
angular velocity of each link (in $\left[-4\pi, 4\pi\right]$ and  $\left[-9\pi, 9\pi \right]$ respectively). The agent
starts with random angles and angular velocities in $\left[-0.1, 0.1 \right]$. The action is the torque applied to the
joint between the two links, in $\left[-1, 1\right]$. The reward is -1 per step.

\subsection{Experimental Details}
\label{sec:exp_details}

We sweep hyperparameters for 40 runs, tuning over the first 10 runs
and reporting results using the final 30 runs for the best hyperparameters. We
sweep critic step size $\alpha = 10^{x}$ for $x \in \left\{ -5, -4, \ldots, -1 \right\}$.
We set the actor step size to be $\kappa \times \alpha$
and sweep $\kappa \in \left\{10^{-3}, 10^{-2}, 10^{-1}, 1
, 2, 10 \right\}$.  We sweep entropy scales $\tau = 10^{y}$ for $y \in \left\{-3, -2, -1, 0, 1\right\}$.
For the classic control experiments, we used fixed batch sizes of 32 samples and a
replay buffer capacity of 100,000 samples.
For the MinAtar experiments, we used fixed batch sizes of 32 samples and a buffer capacity of 1 million. For the Swimmer
experiments, we used fixed batch sizes of 100 samples and a buffer capacity of 1 million.
For CCEM, we fixed $\rho = 0.1$ and sample $N = 30$ actions.

To select hyperparameters across environment, we must normalize performance to provide an aggregate score. We use
near-optimal performance as the normalizer for each environment, with a score of 1 meaning equal to this performance.
We only use this normalization to average scores across environments.
We report learning curves using the original unnormalized returns. For
more details, see Appendix \ref{app_normalization}.

\vspace{-0.25cm}
\subsection{Results}
\vspace{-0.25cm}

\begin{wrapfigure}[14]{}{0.6\textwidth}
	\vspace{-.4cm}
	\includegraphics[width=1.0\linewidth]{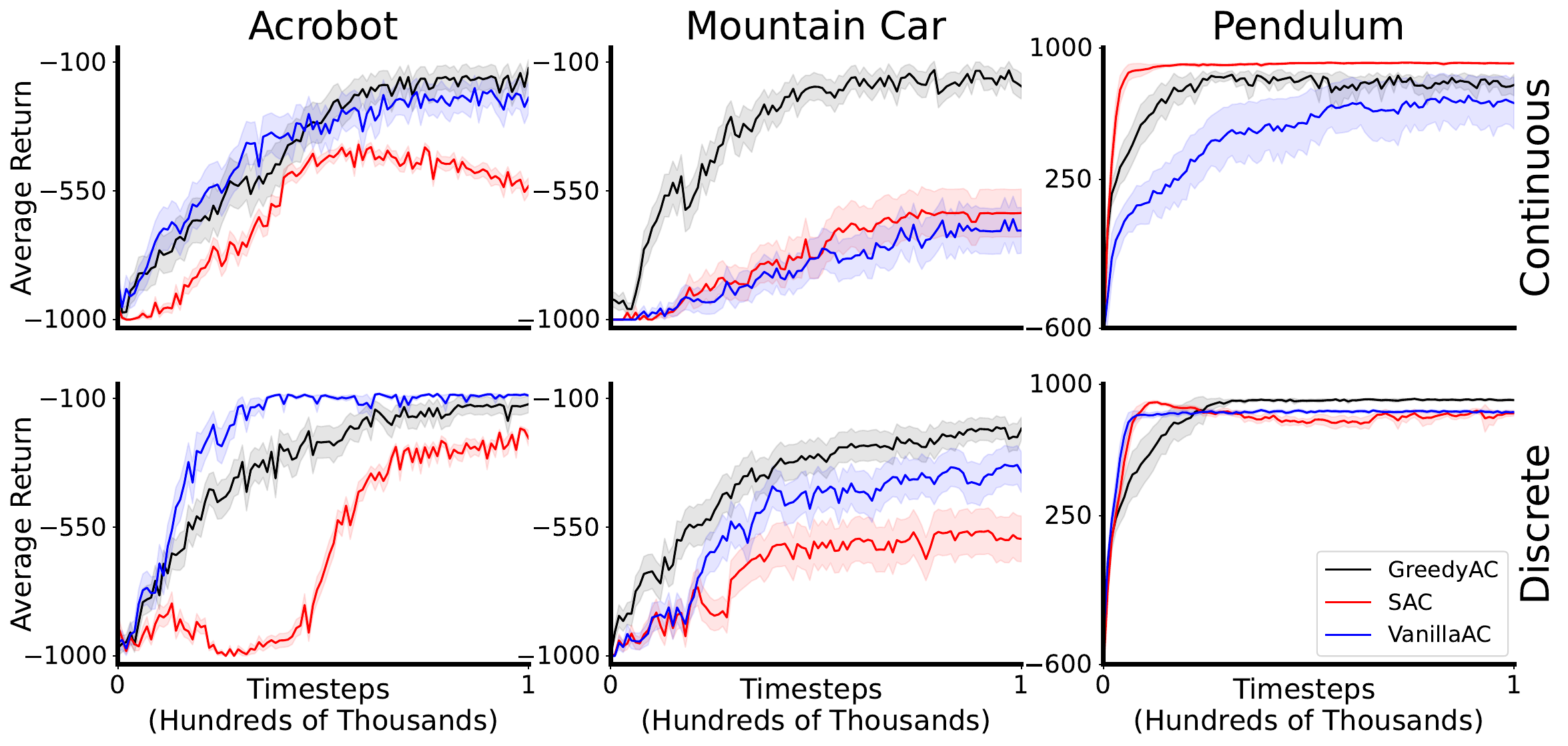}
	\vspace{-.5cm}
	\caption{Learning curves when tuning hyperparameters \textbf{per-environment}, averaged over 30 runs with
	standard errors. }%
	\label{fig:per-env}
\end{wrapfigure}
\textbf{Per-environment Tuning:} We first examine how well the algorithms can perform when they are tuned
per-environment.
In Figure \ref{fig:per-env}, we see that SAC performs well in Pendulum-CA (continuous actions) and in
Pendulum-DA (discrete actions) but poorly in the other settings.
SAC learns slower than GreedyAC and VanillaAC on Acrobot.
GreedyAC performs worse than SAC
in Pendulum-CA, but still performs acceptably, nearly reaching the same final performance. SAC performs poorly
on both versions of Mountain Car. That AC methods struggle
with Acrobot is common wisdom,
but here we see that both GreedyAC and VanillaAC do well on this problem.
GreedyAC is the clear winner in Mountain Car.

\begin{wrapfigure}[14]{}{0.6\textwidth}
	\vspace{-.5cm}
	\centering
	\includegraphics[width=1.0\linewidth]{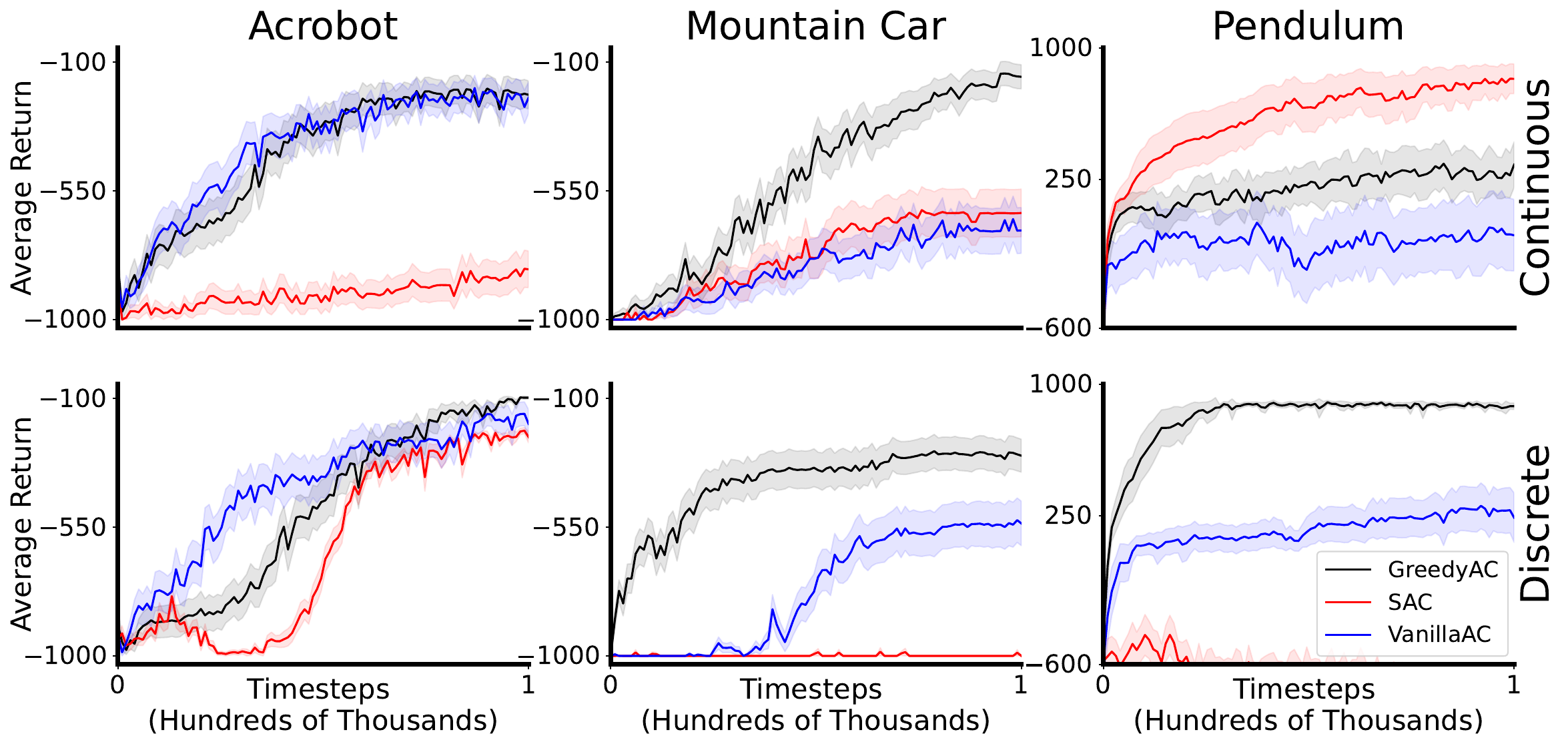}
	\vspace{-.5cm}
	\caption{Learning curves when tuning hyperparameters \textbf{across-environments}, averaged over 30 runs with
	standard errors.}%
	\label{fig:across-env}
\end{wrapfigure}
\textbf{Across-environment Tuning:}
We next examine the performance of the algorithms when they are forced to select one hyperparameter setting across
continuous- or discrete-action environments separately, shown in Figure~\ref{fig:across-env}.
We expect algorithms that are less sensitive to their
parameters to suffer less degradation. Under this regime, GreedyAC has a clear advantage over SAC. GreedyAC
maintains acceptable performance across all environments, sometimes learning more slowly than under per-environment
tuning, but having reasonable behavior. SAC performs poorly on two-thirds of the environments.
GreedyAC is less sensitive than VanillaAC under
across-environment tuning and performs at least as good as VanillaAC.

\textbf{Hyperparameter Sensitivity:} We examine the sensitivity of GreedyAC and SAC
to their entropy scales, focusing on the continuous action environments.
We plot sensitivity curves, with one plotted for each entropy
scale, with the stepsize on the x-axis and average return across all steps and all 40 runs on the y-axis.
Because there are two
stepsizes, we have two sets of plots. When examining the
sensitivity to the critic stepsize, we select the best actor stepsize. We do the same for the actor
stepsize plots. We provide the plots with individual lines in Appendix \ref{app_sensitivity}
and here focus on a more summarized view.

\begin{wrapfigure}{}{0.6\textwidth}
	\vspace{-0.5cm}
	\centering
	\includegraphics[width=1.0\linewidth]{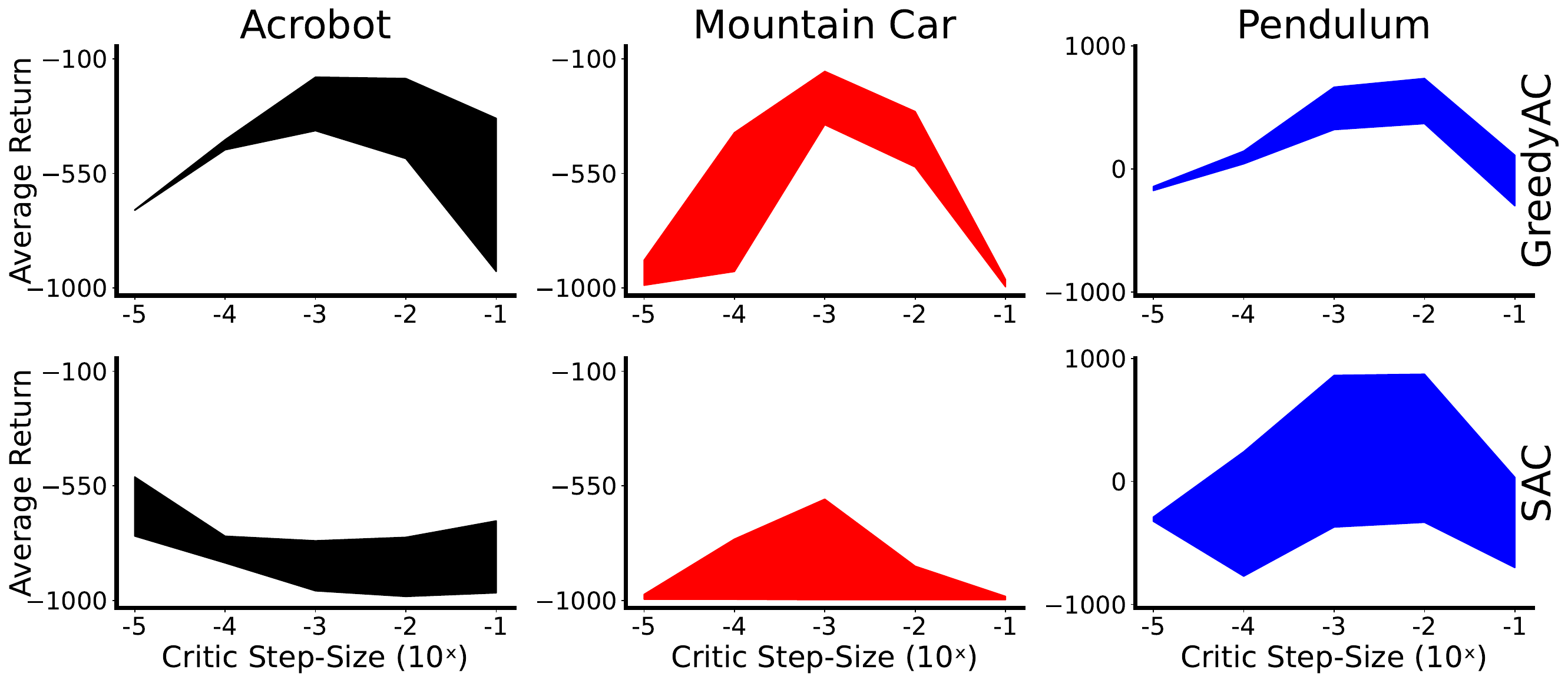}
	\caption{
		A \textbf{sensitivity region} plot for entropy, for GreedyAC (top row) and SAC (bottom row) in the
		continuous action problems.
		\vspace{-.25cm}
}
	\label{fig:summary-sensitivity}
\end{wrapfigure}

Figure \ref{fig:summary-sensitivity} depicts the range of performance obtained across entropy scales. The plot is
generated by filling in the region between the curves for each entropy scale.
If this \emph{sensitivity region} is broad, then the algorithm performed
very differently across different entropy scales and so is sensitive to the entropy.
SAC has much wider sensitivity regions than GreedyAC. Those of GreedyAC are generally narrow,
indicating that the stepsize rather than entropy was the dominant factor.
Further, the bands of performance are generally at the top of the plot. When SAC exhibits narrower regions than
GreedyAC, those regions are lower on the plot, indicating overall poor performance.

\vspace{-0.25cm}

\section{Scaling Greedy-AC}
\begin{figure*}[ht]
		\vspace{-.5cm}
	\centering
 	\begin{subfigure}[b]{0.45\textwidth}
	\includegraphics[width=1.0\linewidth]{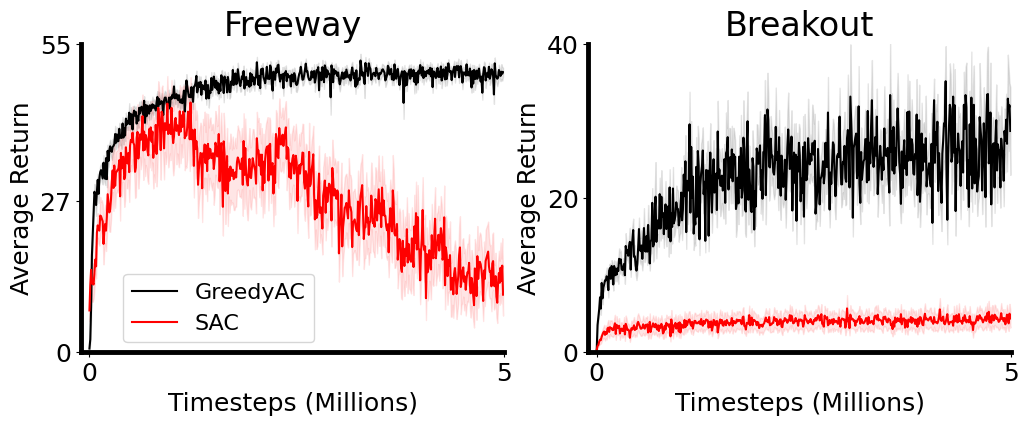}
	\caption{Pixel-based control on MinAtar}%
	\label{fig:minatar}
	\vspace{.244cm}
	\end{subfigure}
\begin{subfigure}[b]{0.45\textwidth}
	\includegraphics[width=1.0\linewidth]{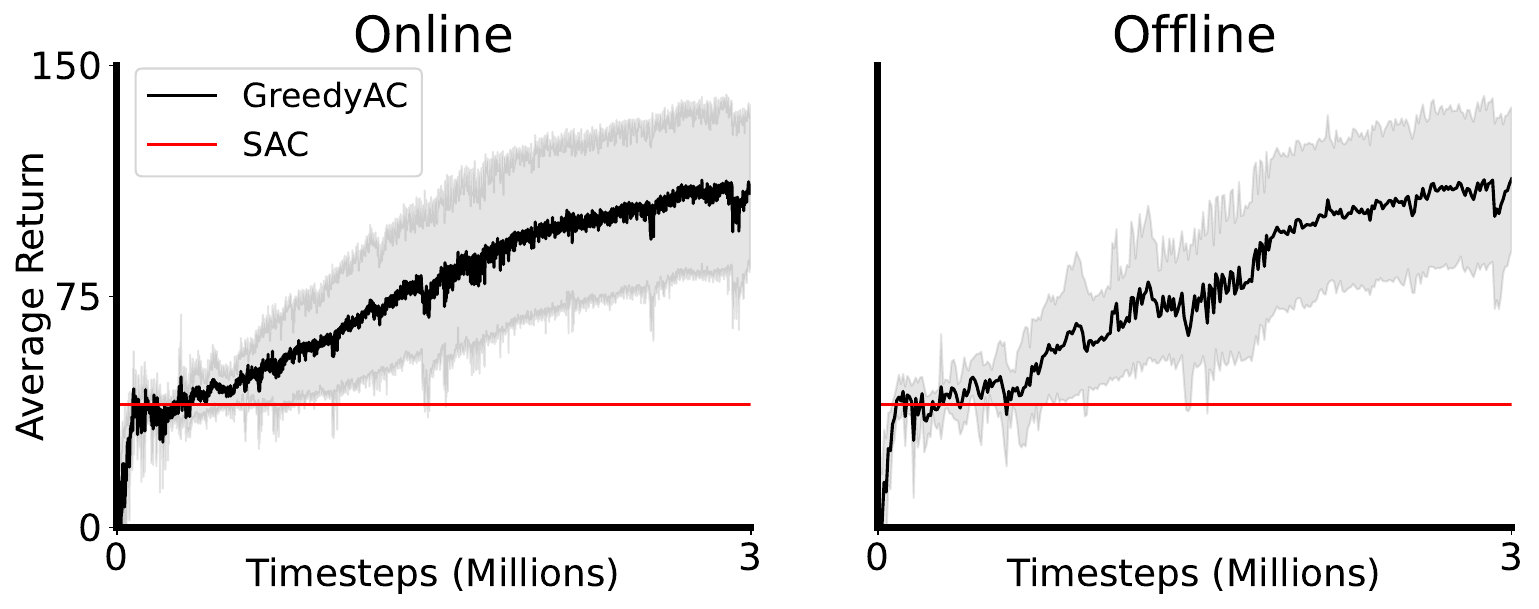}
	\caption{Learning curves on Swimmer over 10 runs, red line denotes SAC's final performance}%
	\label{fig:swimmer}
	\end{subfigure}
		\vspace{-.25cm}
	\caption{Comparing GreedyAC and SAC in more challenging environments}
	\vspace{-.35cm}
\end{figure*}

%

We ran GreedyAC on two pixel-based control problems, Breakout and Freeway, from the MinAtar suite
\citep{young19minatar}. Recent work has shown that MinAtar results are indicative of those in much larger scale problems
\citep{obando2021revisiting}. We set the actor step-size scale to 1.0 and a critic step-size of $10^{-3}$ for both
GreedyAC and SAC---the
defaults of SAC. We set the entropy scale of SAC to $10^{-3}$ based on a grid search. Figure~\ref{fig:minatar} above
clearly indicates GreedyAC can learn a good policy from high-dimensional inputs; comparable performance to DQN Rainbow.

Finally, we ran GreedyAC on Swimmer-v3 from OpenAI Gym \citep{aigym}. We tuned over one run and then ran
the tuned hyperparameters for an additional 9 runs to generate Figure~\ref{fig:swimmer}. We report online
and offline performance. Offline evaluation is performed every
10,000 steps, for 10 episodes, where only the mean action is selected and learning is disabled.
We report SAC's final performance on Swimmer from the SpinningUp benchmark\footnote{See
\url{https://spinningup.openai.com/en/latest/spinningup/bench.html}}. GreedyAC is clearly not state-of-the-art
here---most methods are not---however, GreedyAC steadily improves throughout the experiment.

\vspace{-0.3cm}
\section{Conclusion}%
\label{sec:conclusion}
\vspace{-0.25cm}

In this work, we introduced a new Actor-Critic (AC) algorithm called GreedyAC, that uses a new update to the Actor
based on an extension of the cross-entropy method (CEM). The idea is to (a) define a
\emph{percentile-greedy} target policy and (b) update the actor towards this target policy,
by reducing a KL divergence to it.
This percentile-greedy policy guarantees policy improvement, and we prove that our
Conditional CEM algorithm tracks the actions of maximal value under changing action-values.
We conclude with an in-depth
empirical study, showing that GreedyAC has significantly lower sensitivity to its
hyperparameters than SAC does.

\bibliography{paper}
\bibliographystyle{style/iclr2023_conference}

\appendix

\section{Related Policy Optimization Algorithms}\label{app_connections}

As mentioned in the main text, there are many policy optimization algorithms that can be seen as approximate policy
iteration (API) rather than performing gradient descent on a policy objective. An overview and survey
are given by \citet{vieillard2020leverage} and \citet{chan2021greedification}. There, many methods are shown to either
minimize a forward or reverse KL-divergence to the Boltzmann policy. Our approach similarly updates the actor using a
KL-divergence to a target policy, but here that target policy is the percentile-greedy policy. By doing a maximum
likelihood update with actions sampled under the percentile-greedy policy, we are reducing the forward KL-divergence to
the percentile-greedy policy in Equation \ref{eq_percentile}.

Our CCEM update for the actor is new, but there are several approaches that resemble the idea, particularly those that try to match an expert. This includes dual policy iteration methods (DPI) \citep{parisotto2016actor,sun2018dual,steckelmacher2019sample} and RL as classification methods \citep{lagoudakis2003reinforcement,lazaric2010analysis,farahmand2015classificationbased}.
DPI has two policies, one which is guiding the other. For example, one policy might be an expensive tree search and another a learned neural network, trained to mimic the first (expert or guide) policy. CCEM, on the other hand, uses two policies differently. Our actor does not imitate our proposal policy. Rather, the proposal policy is used to improve the search over the nonconcave surface of Q. It samples actions more broadly, to make it more likely to find a maximizing action. Further, the actor increases the likelihood of only the top actions and does not imitate the proposal policy.
In contrast, Bootstrap DPI \citep[Equation 5]{steckelmacher2019sample} uses an update based on Actor-Mimic \citep{parisotto2016actor}, where the policy increases likelihood of actions for the softmax policies it is trying to mimic.
The resemblance arises from the fact that \citep[Equation 5]{steckelmacher2019sample} can be seen as a sum over forward KL divergences to softmax policies (for discrete actions), just like we have a forward KL divergence but to the percentile-greedy policy (for discrete or continuous actions).

The other class of algorithms, RL as classification, also look similar due to using a forward KL divergence. They reduce the problem to identifying ``positive'' actions in a state (producing maximal returns) and ``negative" actions in a state (producing non-maximal returns). If a cross-entropy loss is used, then this corresponds to maximizing the likelihood of the positive actions and minimizing the likelihood of the negative ones. More generally, other classification algorithms can be used that do not involve maximizing likelihood (like SVMs). The RL as classifications algorithms primarily focus on how to obtain these positive and negative actions, and otherwise look quite different from Greedy AC, in addition to being restricted to a discrete set of actions.

Finally, we can also consider the connection to Conservative Policy Iteration (CPI) \citep{kakade2002approximately} and
a generalization called Deep CPI \citep{vieillard2019deep}. CPI updates the policy to be an interpolation between the
greedy policy $G(Q)$ and the current policy $\pi$, to get the new policy $\pi' = (1-\alpha) \pi + \alpha G(Q)$ for
$\alpha \in [0,1]$. Deep CPI extends this idea to parameterized policies, instead minimizing a forward KL to this
interpolation policy. Greedy AC could be seen as another way to obtain a conservative update, because it does not move
the actor all the way to the greedy policy. Instead, it moves towards the percentile-greedy policy (in Equation
\ref{eq_percentile}), which shifts probability to the upper percentile of actions. Similarly to the interpolation
policy, this percentile-greedy policy depends on the previous policy and on increasing probability for maximally valued
actions. As yet, Deep CPI has not been shown to enjoy the same theoretical guarantees as CPI: minimizing a forward KL to
the interpolation policy does not provide the same guarantees. It remains an open question how to implement this
conservative update in deep RL, and it would be interesting to understand if the CCEM update could provide an
alternative route to obtaining such guarantees.

\section{Convergence Analysis of the Actor} \label{app_theory_main}

We provided an informal proof statement and proof sketch in Section \ref{sec_ccem_theory}, to provide intuition for the
result. Here, we provide the formal proof in the following subsections. We first provide some definitions, particularly
for the quantile function which is central to the analysis. We then lay out the assumptions, and discuss some policy
parameterizations to satisfy those assumptions. We finally state the theorem, with proof, and provide one lemma needed
to prove the theorem in the final subsection.

\subsection{Notation and Definitions}

\paragraph{Notation: }
For a set $A$, let $\mathring{A}$ represent the interior of $A$, while $\partial A$ is the boundary of $A$.
The abbreviation $a.s.$ stands for \textit{almost surely} and $i.o.$ stands for \textit{infinitely often}.
Let $\mathbb{N}$ represent the set $\{0,1,2,\dots\}$. For a set $A$, we let $I_{A}$ to be the indicator
function/characteristic function of $A$ and is defined as $I_{A}(x) = 1$ if $x \in A$ and 0 otherwise. Let
$\mathbb{E}_{g}[\cdot]$, $\mathbb{V}_{g}[\cdot]$ and $\mathbb{P}_{g}(\cdot)$ denote the expectation, variance
and probability measure \emph{w.r.t.} $g$. For a $\sigma$-field $\mathcal{F}$, let $\mathbb{E}\left[\cdot |
\mathcal{F}\right]$ represent the conditional expectation \emph{w.r.t.} $\mathcal{F}$.
A function $f:X \rightarrow Y$ is called Lipschitz continuous if $\exists L \in (0, \infty)$ $s.t.$
$\Vert f(\xvec_1) - f(\xvec_2) \Vert \leq L\Vert \xvec_1 - \xvec_2 \Vert$, $\forall \xvec_1, \xvec_2 \in X$. A
function $f$ is called locally Lipschitz continuous if for every $\xvec \in X$, there exists a neighbourhood $U$
of $X$ such that $f_{\vert U}$  is Lipschitz continuous. Let $C(X, Y)$ represent the space of continuous
functions from $X$ to $Y$. Also, let $B_{r}(\xvec)$ represent an open ball of radius $r$ with centered at $\xvec$.
For a positive integer $M$, let $[M] \defeq \{1,2\dots M\}$.

\vspace{0.2cm}
\begin{defn}\label{def:def2}
	A function $\Gamma:U \subseteq \bbbr^{d_1} \rightarrow V  \subseteq \bbbr^{d_2}$ is Frechet differentiable at
	$\xvec \in U$   if there exists a bounded linear operator $\widehat{\Gamma}_{\xvec}:\bbbr^{d_1} \rightarrow
	\bbbr^{d_2}$ such that the limit
	\begin{flalign}\label{eq:frechetdef}
		\lim_{\epsilon \downarrow 0}\frac{\Gamma(\xvec+\epsilon \yvec) - \xvec}{\epsilon}
	\end{flalign}
	exists and is equal to $\widehat{\Gamma}_{\xvec}(\yvec)$. We say $\Gamma$ is Frechet differentiable if Frechet
	derivative of $\Gamma$ exists at every point in its domain.
\end{defn}

\vspace{0.2cm}
\begin{defn}\label{def:def3}
	Given a bounded real-valued continuous function $H:\bbbr^{d} \rightarrow \bbbr$  with $H(a) \in [H_{l}, H_{u}]$
	and a scalar $\rho \in [0,1]$, we define the $(1-\rho)$-quantile of $H(A)$ \emph{w.r.t.} the PDF $g$ (denoted
	as $f^{\rho}(H, g)$) as follows:
	\begin{flalign}\label{eq:quantdef}
		f^{\rho}(H, g) \defeq \sup_{\ell \in [H_{l}, H_{u}]}\{\mathbb{P}_{g}\big(H(A) \geq \ell\big) \geq \rho\},
	\end{flalign}
	where $\mathbb{P}_{g}$ is the probability measure induced by the PDF $g$, \emph{i.e.}, for a Borel set
	$\mathcal{A}$, $\mathbb{P}_{g}(\mathcal{A}) \defeq \int_{\mathcal{A}}g(a)da$.
\end{defn}

This quantile operator will be used to succinctly write the quantile for $Q_{\theta}(S, \cdot)$, with actions
selected according to $\pi_{\pparams}$, i.e.,
\begin{equation}
	f_{\theta}^{\rho}(\pparams; s) \defeq f^{\rho}(Q_{\theta}(s, \cdot), \pi_{\pparams}(\cdot \vert s)) =
	\sup_{\ell \in [Q^{\theta}_{l}, Q^{\theta}_{u}]}\{\mathbb{P}_{\pi_{\pparams}(\cdot|s)}\big(Q_{\theta}(s,A)
	\geq \ell\big) \geq \rho\} \label{eq:quantile}
 .
 \end{equation}

\subsection{Assumptions}

\begin{assm}\label{assm:assm1}
Given a realization of the transition dynamics of the MDP in the form of a sequence of transition tuples $\mathcal{O} \defeq \{(S_{t}, A_{t}, R_{t}, S^{\prime}_{t})\}_{t \in \mathbb{N}}$, where the state $S_t \in \States$ is drawn using a latent sampling distribution $\nu$, while $A_{t} \in \Actions$ is the action chosen at state $S_t$, the transitioned state $\States \ni S^{\prime}_{t} \sim P(S_{t}, A_{t}, \cdot)$ and the reward $\bbbr \ni R_{t} \defeq R(S_{t}, A_{t}, S^{\prime}_{t})$. We further assume that the reward is uniformly bounded, \emph{i.e.}, $\vert R(\cdot, \cdot, \cdot) \vert < R_{max} < \infty$.
\end{assm}
We analyze the long run behaviour of the conditional cross-entropy recursion (actor) which is defined as follows:
\begin{flalign}
	&\pparams_{t+1} \defeq \Gamma^{W}\left\{\pparams_{t} + \alpha_{a, t} \frac{1}{N_{t}}\sum_{A \in \Xi_t}{I_{\{Q_{\theta_t}(S_{t}, A) \geq \widehat{f}^{\rho}_{t+1}\}}}  \nabla_{\pparams_t}  \ln \pi_\pparams(A|S_{t})\right\},\label{eq:pparamsrec}\\
	&\hspace*{30mm}\text { where } \Xi_{t} \defeq \{A_{t,1}, A_{t,2}, \dots, A_{t,N_t}\} \mysim\pi_{\pparams^{\prime}_{t}}(\cdot | S_{t}).\nonumber\\\nonumber\\
	&\pparams^{\prime}_{t+1} \defeq \pparams^{\prime}_{t} + \alpha^{\prime}_{a, t}\left(\pparams_{t+1} - \pparams^{\prime}_{t}\right).\label{eq:pparams_prime_rec}
\end{flalign}
Here, $\Gamma^{W}\{\cdot\}$ is the projection operator onto the compact (closed and bounded) and convex set $W \subset \bbbr^{m}$  with a
smooth boundary $\partial W$. Therefore, $\Gamma^{W}$ maps vectors in  $\bbbr^{m}$ to the nearest vectors in  $W$  \emph{w.r.t.} the
Euclidean distance (or equivalent metric). Convexity and compactness ensure that the projection is
unique and belongs to $W$.
\begin{assm}\label{assm:assm2}
The pre-determined, deterministic, step-size sequences $\{\alpha_{a, t}\}_{t \in \mathbb{N}}$, $\{\alpha^{\prime}_{a, t}\}_{t \in \mathbb{N}}$ and $\{\alpha_{q, t}\}_{t \in \mathbb{N}}$  are positive scalars which satisfy the following:
\begin{flalign*}
&\sum_{t \in \mathbb{N}}\alpha_{a, t} = \sum_{t \in \mathbb{N}}\alpha^{\prime}_{a, t} = \sum_{t \in \mathbb{N}}\alpha_{q, t} = \infty\\
& \sum_{t \in \mathbb{N}}\left(\alpha^{2}_{a, t} + {\alpha^{\prime}}^{2}_{a, t} + \alpha^{2}_{q, t}\right) < \infty\\
&\lim_{t \rightarrow \infty}\frac{\alpha^{\prime}_{a, t}}{\alpha_{a, t}} = 0, \hspace*{4mm} \lim_{t \rightarrow \infty}\frac{\alpha_{q, t}}{\alpha_{a, t}} = 0.
\end{flalign*}
\end{assm}
The first conditions in Assumption \ref{assm:assm2} are the classical Robbins-Monro conditions \citep{robbins1985stochastic} required for stochastic approximation algorithms. The last two conditions enable the different stochastic recursions to have separate timescales. Indeed, it ensures the $\pparams_t$ recursion is faster compared to the recursions of $\theta_t$ and $\pparams^{\prime}_{t}$. This timescale divide is needed to obtain the desired asymptotic behaviour, as we describe in the next section.
\begin{assm}\label{assm:assm3}
The pre-determined, deterministic, sample length schedule $\{N_{t} \in \mathbb{N}\}_{t \in \mathbb{N}}$ is positive and strictly monotonically increases to $\infty$ and $\inf_{t \in \mathbb{N}}\frac{N_{t+1}}{N_{t}} > 1$.
\end{assm}
Assumption \ref{assm:assm3} states that the number of samples increases to infinity and is primarily required to ensure that the estimation error arising due to the estimation of sample quantiles eventually decays to $0$. Practically, one can indeed consider a fixed, finite, positive integer for $N_t$ which is large enough to accommodate the acceptable error.
\begin{assm}\label{assm:assm4}
	The sequence $\{\theta_{t}\}_{t \in \mathbb{N}}$ satisfies $\theta_t \in \Theta$, where $\Theta$ $\subset \bbbr^{n}$ is a convex, compact set. Also, for $\theta \in \Theta$, let $Q_{\theta}(s, a) \in [Q^{\theta}_{l}, Q^{\theta}_{u}]$, $\forall s \in \States, a \in \Actions$.
\end{assm}
Assumption \ref{assm:assm4} assumes stability of the Expert, and minimally only requires that the values remain in a bounded range. We make no additional assumptions on the convergence properties of the Expert, as we simply need stability to prove the Actor tracks the update.
\begin{assm}\label{assm:assm5}
	For $\theta \in \Theta$ and $s \in \States$, let $\mathbb{P}_{A \sim \pi_{\pparams^{\prime}}(\cdot \vert s)}\left(Q_{\theta}(s, A) \geq \ell\right) > 0$, $\forall \ell \in [Q^{\theta}_{l}, Q^{\theta}_{u}]$ and $\forall \pparams^{\prime} \in W$.
\end{assm}
Assumption \ref{assm:assm5} implies that there always exists a strictly positive probability mass beyond every threshold $\ell \in [Q^{\theta}_{l}, Q^{\theta}_{u}]$. This assumption is easily satisfied when $Q_{\theta}(s, a)$ is continuous in $a$ and $\pi_{\pparams}(\cdot \vert s)$ is a continuous probability density function.
\begin{assm}\label{assm:assm6}
	\begin{flalign*}
		&\sup_{\substack{\pparams, \pparams^{\prime} \in W, \\ \theta \in \Theta, \ell \in [Q^{\theta}_{l}, Q^{\theta}_{u}]}}
		\mathbb{E}_{A \sim \pi_{\pparams^{\prime}}(\cdot \vert S)}\Big[\Big\Vert I_{\{Q_{\theta}(S, A) \geq \ell\}}\nabla_{\pparams}\ln{\pi_{\pparams}(A \vert S)} - \\[-20pt]
		&\hspace*{40mm}\mathbb{E}_{A \sim \pi_{\pparams^{\prime}}(\cdot \vert S)}\left[I_{\{Q_{\theta}(S, A) \geq \ell\}}\nabla_{\pparams}\ln{\pi_{\pparams}(A \vert S)} \big\vert S\right] \Big\Vert_{2}^{2} \Big\vert S\Big] < \infty \hspace*{2mm} a.s.,\\
		&\sup_{\substack{\pparams, \pparams^{\prime} \in W, \\ \theta \in \Theta, \ell \in [Q^{\theta}_{l}, Q^{\theta}_{u}]}}  \mathbb{E}_{A \sim \pi_{\pparams^{\prime}}(\cdot \vert S)}\left[\Big\Vert I_{\{Q_{\theta}(S, A) \geq \ell\}}  \nabla_{\pparams}  \ln \pi_{\pparams}(A | S) \Big\Vert_{2}^{2} \Big\vert S \right] < \infty \hspace*{2mm} a.s.
	\end{flalign*}
\end{assm}
\begin{assm}\label{assm:assm7}
	For $s \in \States$, $\nabla_{\pparams}\ln{\pi_{\pparams}(\cdot \vert s)}$ is locally Lipschitz continuous w.r.t. $\pparams$.
\end{assm}
Assumptions \ref{assm:assm6} and \ref{assm:assm7} are technical requirements that can be more easily characterized when we consider $\pi_{\pparams}$ to belong to the natural exponential family (NEF) of distributions.
\begin{defn}\label{def:def1}
\textbf{Natural exponential family of distributions (NEF)\citep{morris1982natural}: }These probability distributions over $\bbbr^{m}$ are represented by
\begin{equation}\label{eqn:nefdef}
 \{\pi_{\eta}(\xvec) \defeq h(\xvec)e^{\eta^{\top}T(\xvec)-K(\eta)} \mid \eta \in \Lambda \subset \bbbr^d\},
\end{equation}
	where $\eta$ is the natural parameter, $h:\bbbr^{m} \longrightarrow \bbbr$, while $T:\bbbr^{m} \longrightarrow \bbbr^{d}$ (called the \textit{sufficient statistic}) and $K(\eta) \defeq \ln{\int{h(\xvec)e^{\eta^{\top}T(\xvec)}d\xvec}}$ (called the \textit{cumulant function} of the family). The space $\Lambda$ is  defined as $\Lambda \defeq \{\eta \in \bbbr^{d} \vert \hspace*{3mm}\vert K(\eta) \vert < \infty\}$. Also, the above representation is assumed minimal.\footnote{For a distribution in NEF, there may exist multiple representations of the form (\ref{eqn:nefdef}). However, for the distribution, there definitely exists a representation where the components of the sufficient statistic are linearly independent and such a representation is referred to as \textit{minimal}.}
A few popular distributions which belong to the NEF family include Binomial, Poisson, Bernoulli, Gaussian, Geometric and Exponential distributions.
\end{defn}
We parametrize the policy $\pi_{\pparams}(\cdot \vert S)$ using a neural network, which implies that when we consider NEF for the stochastic policy, the natural parameter $\eta$ of the NEF is being parametrized by $\pparams$. To be more specific, we have $\{\psi_{\pparams}:\States \rightarrow \Lambda | \pparams \in \bbbr^{m}\}$ to be the function space induced by the neural network of the actor, \emph{i.e.}, for a given state $s \in \States$, $\psi_{\pparams}(s)$ represents the natural parameter of the NEF policy $\pi_{\pparams}(\cdot \vert s)$. Further,
\begin{flalign}
	\nabla_{\pparams}  \ln \pi_\pparams(A|S)
	& = \ln{(h(A))} + \psi_{\wvec}(S_{t})^{\top}T(A) - K(\psi_{\wvec}(S))\nonumber\\
	&=  \nabla_{\pparams}\psi_{\pparams}(S) \left(T(A) - \nabla_{\eta}K(\psi_{\pparams}(S))\right).\nonumber\\
	&=  \nabla_{\pparams}\psi_{\pparams}(S) \left(T(A) - \mathbb{E}_{A \sim \pi_{\pparams}(\cdot \vert S)}\left[T(A)\right]\right).
\end{flalign}
Therefore Assumption \ref{assm:assm7} can be directly satisfied by assuming that $\psi_{w}$ is twice continuously differentiable \emph{w.r.t.} $\pparams$.

\begin{assm}\label{assm:assm9}
	For every $\theta \in \Theta$, $s \in \States$ and $\pparams \in W$, $f_{\theta}^{\rho}(\pparams; s)$ (from Eq. \eqref{eq:quantile}) exists and is unique.
\end{assm}
The above assumption ensures that the true $(1-\rho)$-quantile is unique and the assumption is usually satisfied for most distributions and a well-behaved $Q_{\theta}$.

\subsection{Main Theorem}

To analyze the algorithm, we  employ here the ODE-based analysis as proposed in
\citep{borkar2008stochastic,kushner2012stochastic}.  The actor recursions (Eqs.
(\ref{eq:pparamsrec}-\ref{eq:pparams_prime_rec})) represent a classical two timescale stochastic approximation
recursion, where there exists a bilateral coupling between the individual stochastic recursions (\ref{eq:pparamsrec})
and (\ref{eq:pparams_prime_rec}). Since the step-size schedules $\{\alpha_{a,t}\}_{t \in \mathbb{N}}$ and
$\{\alpha^{\prime}_{a, t}\}_{t \in \mathbb{N}}$ satisfy $\frac{\alpha^{\prime}_{a,t}}{\alpha_{a,t}} \rightarrow 0$, we
have $\alpha^{\prime}_{a,t} \rightarrow 0$ relatively faster than $\alpha_{a,t} \rightarrow 0$. This disparity induces a
pseudo-heterogeneous rate of convergence (or timescales) between the individual stochastic recursions which further
amounts to the asymptotic emergence of a stable coherent behaviour which is quasi-asynchronous.  This pseudo-behaviour
can be interpreted using multiple viewpoints. When viewed from the faster timescale recursion--- controlled by
$\alpha_{a, t}$---the slower timescale recursion---controlled by $\alpha^{\prime}_{a, t}$---appears quasi-static, i.e.,
almost a constant. Likewise, when observed from the slower timescale, the faster timescale recursion seems equilibrated.

The existence of this stable long run behaviour under certain standard assumptions of stochastic approximation
algorithms is rigorously established in \citep{borkar1997stochastic} and also in Chapter 6 of
\citep{borkar2008stochastic}. For our stochastic approximation setting (Eqs.
(\ref{eq:pparamsrec}-\ref{eq:pparams_prime_rec})), we can directly apply this appealing characterization of the long run
behaviour of the two timescale stochastic approximation algorithms---after ensuring the compliance of our setting to the
pre-requisites demanded by the characterization---by considering the slow timescale stochastic recursion
(\ref{eq:pparams_prime_rec}) to be quasi-stationary (\emph{i.e.}, ${\pparams^{\prime}}_t \equiv \pparams^{\prime}$,
$a.s.$, $\forall t \in \mathbb{N}$), while analyzing the limiting behaviour of the  faster timescale recursion
(\ref{eq:pparamsrec}). Similarly, we let $\theta_t$ to be quasi-stationary too (\emph{i.e.}, $\theta_t \equiv \theta$,
$a.s.$, $\forall t \in \mathbb{N}$). The asymptotic behaviour of the slower timescale recursion is further analyzed by
considering the faster timescale temporal variable $\pparams_t$ with the limit point so obtained during quasi-stationary
analysis.

Define the filtration $\{\mathcal{F}_{t}\}_{t \in \mathbb{N}}$, a family of increasing natural $\sigma$-fields, where
\begin{equation*}
\mathcal{F}_{t}
\defeq \sigma\left(\{{\pparams}_i, {\pparams}^{\prime}_i, (S_i, A_i, R_i, S^{\prime}_i), \Xi_{i}; 0 \leq i \leq t\}\right)
.
\end{equation*}
\begin{theorem}\label{thm_main}
	Let $\pparams^{\prime}_{t} \equiv \pparams^{\prime}, \theta_{t} \equiv \theta, \forall t \in \mathbb{N}$ $a.s.$
	Let Assumptions \ref{assm:assm1}-\ref{assm:assm9} hold. Then the stochastic sequence $\{\pparams_{t}\}_{t \in \mathbb{N}}$ generated by the stochastic recursion (\ref{eq:pparamsrec}) asymptotically tracks the ODE:
\begin{equation}\label{eq:wt_thm_ode}
	\frac{d}{dt}{\pparams}(t) = \widehat{\Gamma}^{W}_{{\pparams}(t)}\left(\nabla_{\pparams(t)}\mathbb{E}_{\substack{S \sim \nu, A \sim \pi_{\pparams^{\prime}}(\cdot \vert S)}}\Big[I_{\{Q_{\theta}(S, A) \geq f_{\theta}^{\rho}(\pparams^{\prime}; S)\}} \ln \pi_{\pparams(t)}(A|S)\Big]\right), \hspace*{4mm} t \geq 0.
\end{equation}
	In other words, $\lim_{t \rightarrow \infty}\pparams_{t} \in \mathcal{K}$ $a.s.$, where $\mathcal{K}$ is set of stable equilibria of the ODE (\ref{eq:wt_thm_ode}) contained inside $W$.
\end{theorem}
\begin{proof}
Firstly, we rewrite the stochastic recursion (\ref{eq:pparamsrec}) under the hypothesis that $\theta_t$ and $\pparams^{\prime}_{t}$ are quasi-stationary, \emph{i.e.},  $\theta_t \underset{a.s.}{\equiv} \theta$  and $\pparams^{\prime}_{t} \underset{a.s.}{\equiv} \pparams^{\prime}$ as follows:
\begin{equation}\label{eq:wtrec_i_p_first}
	\pparams_{t+1} \defeq \Gamma^{W}\left\{\pparams_{t} + \alpha_{a, t} \frac{1}{N_{t}}\sum_{A \in \Xi_{t}}{I_{\{Q_{\theta}(S_{t}, A) \geq \widehat{f}^{\rho}_{t+1}\}}}  \nabla_\pparams  \ln \pi_\pparams(A|S_{t})\right\}
\end{equation}
where $f_{\theta}^{\rho}(\pparams^{\prime}; S) \defeq f^{\rho}(Q_{\theta}(S, \cdot), \pi_{\pparams^{\prime}}(\cdot \vert S))$ and $\nabla_{\pparams_t} \defeq \nabla_{\pparams = \pparams_t}$, \emph{i.e.}, the gradient \emph{w.r.t.} $\pparams$ at $\pparams_t$. Define
\begin{flalign}
	g^{\theta}(\pparams) &\defeq \mathbb{E}_{S_{t} \sim \nu, A \sim \pi_{\pparams^{\prime}}(\cdot \vert S_{t})}\Big[I_{\{Q_{\theta}(S_{t}, A) \geq f_{\theta}^{\rho}(\pparams^{\prime}; S_{t})\}}  \nabla_{\pparams}  \ln \pi_\pparams(A|S_{t})\Big].&&
\end{flalign}
\begin{flalign}
	\mathbb{M}_{t+1} &\defeq
	\frac{1}{N_{t}}\sum_{A \in \Xi_{t}}{I_{\{Q_{\theta}(S_{t}, A) \geq \widehat{f}^{\rho}_{t+1}\}}  \nabla_{\pparams_{t}}  \ln \pi_\pparams(A|S_{t})}  - &&\nonumber\\
	&\hspace*{10mm}\mathbb{E}\Bigg[\frac{1}{N_{t}}\sum_{A \in \Xi_{t}}{I_{\{Q_{\theta}(S_{t}, A) \geq \widehat{f}^{\rho}_{t+1}\}}  \nabla_{\pparams_{t}}  \ln \pi_\pparams(A|S_{t})} \Big\vert \mathcal{F}_{t}\Bigg].
\end{flalign}
\begin{flalign}
	\ell^{\theta}_{t} &\defeq \mathbb{E}\bigg[\frac{1}{N_{t}}\sum_{A \in \Xi_{t}}{I_{\{Q_{\theta}(S_{t}, A) \geq \widehat{f}^{\rho}_{t+1}\}}  \nabla_{\pparams_t}  \ln \pi_\pparams(A | S_{t})} \bigg\vert \mathcal{F}_{t}\bigg] - &&\nonumber\\
&\hspace*{10mm}\mathbb{E}_{S_{t} \sim \nu, A \sim \pi_{\pparams^{\prime}}(\cdot \vert S_{t})}\Big[I_{\{Q_{\theta}(S_{t}, A) \geq f_{\theta}^{\rho}(\pparams^{\prime}; S_{t})\}}  \nabla_{\pparams_t}  \ln \pi_\pparams(A | S_{t})\Big]
\end{flalign}
Then we can rewrite
\begin{flalign}\label{eq:wtrec_i_p}
	\eqref{eq:wtrec_i_p_first}
	&= \Gamma^{W}\Bigg\{\pparams_{t} + \alpha_{a, t}\Bigg(\mathbb{E}_{S_{t} \sim \nu, A \sim \pi_{\pparams^{\prime}}(\cdot \vert S_{t})}\left[I_{\{Q_{\theta}(S_{t}, A) \geq f_{\theta}^{\rho}(\pparams^{\prime}; S_{t})\}}  \nabla_{\pparams_t}  \ln \pi_\pparams(A|S_{t})\right] - \nonumber\\
	&\hspace*{15mm}\mathbb{E}_{S_{t} \sim \nu, A \sim \pi_{\pparams^{\prime}}(\cdot \vert S_{t})}\Big[I_{\{Q_{\theta}(S_{t}, A) \geq f_{\theta}^{\rho}(\pparams^{\prime}; S_{t})\}}  \nabla_{\pparams_t}  \ln \pi_\pparams(A|S_{t})\Big]  + \nonumber\\
	&\hspace*{20mm}\mathbb{E}\bigg[\frac{1}{N_{t}}\sum_{A \in \Xi_{t}}{I_{\{Q_{\theta}(S_{t}, A) \geq \widehat{f}^{\rho}_{t+1}\}}  \nabla_{\pparams_t}  \ln \pi_\pparams(A|S_{t})} \bigg\vert \mathcal{F}_{t}\bigg] - \nonumber\\
	&\hspace*{25mm}\mathbb{E}\bigg[\frac{1}{N_{t}}\sum_{A \in \Xi_{t}}{I_{\{Q_{\theta}(S_{t}, A) \geq \widehat{f}^{\rho}_{t+1}\}}  \nabla_{\pparams_t}  \ln \pi_\pparams(A | S_{t})} \bigg\vert \mathcal{F}_{t}\bigg] + \nonumber\\
	&\hspace*{30mm}\frac{1}{N_{t}}\sum_{A \in \Xi_{t}}{I_{\{Q_{\theta}(S_{t}, A) \geq \widehat{f}^{\rho}_{t+1}\}}}  \nabla_{\pparams_{t}}  \ln \pi_\pparams(A|S_{t})\Bigg)\Bigg\}.\nonumber\\\nonumber\\
	&=\Gamma^{W}\Big\{g^{\theta}(\pparams_t) + \mathbb{M}_{t+1} + \ell^{\theta}_{t}\Big\},
\end{flalign}

\vspace*{5mm}

\noindent
A few observations are in order:
\begin{enumerate}
	\item[B1.]
	$\{\mathbb{M}_{t+1}\}_{t \in \mathbb{N}}$ is a martingale difference noise sequence \emph{w.r.t.} the filtration $\{\mathcal{F}_{t}\}_{t \in \mathbb{N}}$, \emph{i.e.}, $\mathbb{M}_{t+1}$ is $\mathcal{F}_{t+1}$-measurable and integrable, $\forall t \in \mathbb{N}$ and $\mathbb{E}\left[\mathbb{M}_{t+1} \vert \mathcal{F}_{t}\right] = 0$ $a.s.$, $\forall t \in \mathbb{N}$.
\item[B2.]
	$g^{\theta}$ is locally Lipschitz continuous. This follows from Assumption \ref{assm:assm7}.
\item[B3.]
	$\ell^{\theta}_{t} \rightarrow 0$ $a.s.$ as $t \rightarrow \infty$. (By Lemma \ref{lem:ltheta_conv} below).
\item[B4.]
	The iterates $\{\pparams_{t}\}_{t \in \mathbb{N}}$ is bounded almost surely, \emph{i.e.},
\begin{flalign*}
\sup_{t \in \mathbb{N}} \Vert \pparams_{t} \Vert < \infty \hspace*{4mm} a.s.
\end{flalign*}
		This is ensured by the explicit application of the projection operator $\Gamma^{W}\{\cdot\}$ over the iterates $\{\pparams_t\}_{t \in \mathbb{N}}$ at every iteration onto the bounded set $W$.
\item[B5.]
	$\exists L \in (0,\infty) \hspace*{2mm} s.t. \hspace*{2mm} \mathbb{E}\left[\Vert \mathbb{M}_{t+1} \Vert^{2} \vert \mathcal{F}_{t}\right] \leq L\left(1+\Vert \pparams_{t} \Vert^{2}\right) \hspace*{2mm} a.s.$\vspace*{4mm}\\
This follows from Assumption \ref{assm:assm6} (ii).
\end{enumerate}
\vspace*{10mm}
Now, we rewrite the stochastic recursion (\ref{eq:wtrec_i_p}) as follows:
\begin{flalign}\label{eq:projexpert}
	\pparams_{t+1} &\defeq \pparams_{t} +  \alpha_{a,t}\frac{\Gamma^{W}\left\{\pparams_{t} + \xi_{t}\left(g^{\theta}(\pparams_{t}) + \mathbb{M}_{t+1} + \ell^{\theta}_{t}\right)\right\} - \pparams_{t}}{\alpha_{a, t}}\nonumber\\
&= {\pparams}_{t} + \alpha_{a, t}\left(\widehat{\Gamma}^{W}_{{\pparams}_t}(g^{\theta}({\pparams}_{t})) + \widehat{\Gamma}^{W}_{{\pparams}_t}\left(\mathbb{M}_{t+1}\right) + \widehat{\Gamma}^{W}_{{\pparams}_t}\left({\ell}^{\theta}_{t}\right) + o(\alpha_{a, t})\right),\end{flalign}
where $\widehat{\Gamma}^{W}$ is the Frechet derivative (Definition \ref{def:def1}).

The above stochastic recursion is also a stochastic approximation recursion with the vector field  $\widehat{\Gamma}^{W}_{{\pparams}_t}(g^{\theta}({\pparams}_{t}))$, the noise term $\widehat{\Gamma}^{W}_{{\pparams}_t}\left(\mathbb{M}_{t+1}\right)$, the bias term $\widehat{\Gamma}^{W}_{{\pparams}_t}\left({\ell}^{\theta}_{t}\right)$ with an additional error term $o(\alpha_{a, t})$ which is asymptotically inconsequential.

Also, note that $\Gamma^{W}$ is single-valued map since the set $W$ is assumed convex and also the limit exists since the boundary $\partial W$ is considered smooth. Further, for $\pparams \in \mathring{W}$, we have
\begin{flalign}\label{eq:projimap}
\widehat{\Gamma}^{W}_{\pparams}(\uvec) \defeq \lim_{\epsilon \rightarrow 0}\frac{\Gamma^{W}\left\{\pparams + \epsilon \uvec\right\} - \pparams}{\epsilon} = \lim_{\epsilon \rightarrow 0}\frac{\pparams + \epsilon \uvec - \pparams}{\epsilon}  = \uvec \text{ (for sufficiently small }\epsilon),
\end{flalign}
\emph{i.e.}, $\widehat{\Gamma}^{W}_{\pparams}(\cdot)$ is an identity map for $\pparams \in \mathring{W}$.

Now by appealing to Theorem 2, Chapter 2 of \citep{borkar2008stochastic} along with the observations B1-B5, we conclude that the stochastic recursion (\ref{eq:pparamsrec}) asymptotically tracks the following ODE almost surely:
\begin{flalign}\label{eq:projexpert_ode}
\frac{d}{dt}{\pparams}(t) &= \widehat{\Gamma}^{W}_{{\pparams}(t)}(g^{\theta}({\pparams}(t))), \hspace*{4mm} t \geq 0\nonumber\\
&= \widehat{\Gamma}^{W}_{{\pparams}(t)}\left(\mathbb{E}_{\substack{S \sim \nu, A \sim \pi_{\pparams^{\prime}}(\cdot \vert S)}}\Big[I_{\{Q_{\theta}(S, A) \geq f_{\theta}^{\rho}(\pparams^{\prime}; S)\}}  \nabla_{\pparams(t)}  \ln \pi_\pparams(A|S)\Big]\right)\nonumber\\
&= \widehat{\Gamma}^{W}_{{\pparams}(t)}\left(\nabla_{\pparams(t)}\mathbb{E}_{\substack{S \sim \nu, A \sim \pi_{\pparams^{\prime}}(\cdot \vert S)}}\Big[I_{\{Q_{\theta}(S, A) \geq f_{\theta}^{\rho}(\pparams^{\prime}; S)\}} \ln \pi_\pparams(A|S)\Big]\right).
\end{flalign}
The interchange of expectation and the gradient in the last equality follows from dominated convergence theorem and
Assumption \ref{assm:assm7} \citep{rubinstein1993discrete}.  The above ODE is a gradient flow  with dynamics restricted
inside $W$. This further implies that the stochastic recursion (\ref{eq:pparamsrec}) converges to a (possibly sample
path dependent) asymptotically stable equilibrium point of the above ODE inside $W$.
\end{proof}

\subsection{Proof of Lemma \ref{lem:ltheta_conv} to satisfy Condition 3}

In this section, we show that $\ell^{\theta}_{t} \rightarrow 0$ $a.s.$ as $t \rightarrow \infty$, in Lemma
\ref{lem:ltheta_conv}. To do so, we first need to prove several supporting lemmas. Lemma \ref{lem:quant_conv} shows
that, for a given Actor and Expert, the sample quantile converges to the true quantile. Using this lemma, we can then
prove Lemma \ref{lem:ltheta_conv}. In the following subsection, we provide three supporting lemmas about convexity and
Lipschitz properties of the sample quantiles, required for the proof Lemma \ref{lem:quant_conv}.

For this section, we require the following characterization of $f^{\rho}(Q_{\theta}(s, \cdot), \pparams^{\prime})$. Please refer Lemma 1 of \citep{homem2007study} for more details.
\begin{flalign}\label{eq:quant_chn}
	f^{\rho}(Q_{\theta}(s, \cdot), \pparams^{\prime}) = \argmin_{\ell \in [ Q^{\theta}_{l}, Q^{\theta}_{u} ]}{\mathbb{E}_{A \sim \pi_{\pparams^{\prime}}(\cdot \vert s)}\left[\Psi(Q_{\theta}(s, A), \ell)\right]},
\end{flalign}
where $\Psi(y, \ell) \defeq (y - \ell)(1-\rho)I_{\{y \geq \ell \}} +  (\ell - y)\rho I_{\{\ell \geq y\}}$.\\

Similarly, the sample estimate of the true $(1-\rho)$-quantile, \emph{i.e.}, $\widehat{f}^{\rho} \defeq Q^{(\lceil(1-\rho)N\rceil)}_{\theta, s}$, (where $Q_{\theta, s}^{(i)}$ is the $i$-th order statistic of the random sample $\{Q_{\theta}(s, A)\}_{A \in \Xi}$ with $\Xi \defeq \{A_i\}_{i=1}^{N} \mysim \pi_{\pparams^{\prime}}(\cdot \vert s)$) can be characterized as the unique solution of the stochastic counterpart of the above optimization problem, \emph{i.e.},
\begin{flalign}\label{eq:quant_hat}
	\widehat{f}^{\rho} = \argmin_{\ell \in [ Q^{\theta}_{l}, Q^{\theta}_{u} ]}{\frac{1}{N}\sum_{\substack{A \in \Xi \\ \vert \Xi \vert = N}}\Psi(Q_{\theta}(s, A), \ell)}.
\end{flalign}

\begin{lem}\label{lem:quant_conv}
	Assume $\theta_t \equiv \theta$, $\pparams^{\prime}_{t} \equiv \pparams^{\prime}$, $\forall t \in \mathbb{N}$. Also, let Assumptions 3-5 hold. Then, for a given state $s \in \States$,
	\begin{flalign*}
		\lim_{t \rightarrow \infty}\widehat{f}^{\rho}_{t} = f^{\rho}(Q_{\theta}(s, \cdot), \pparams^{\prime}) \hspace*{2mm} a.s.,
	\end{flalign*}
	where $\widehat{f}^{\rho}_{t} \defeq Q^{(\lceil(1-\rho)N_t\rceil)}_{\theta, s}$, (where $Q_{\theta, s}^{(i)}$ is the $i$-th order statistic of the random sample $\{Q_{\theta}(s, A)\}_{A \in \Xi_{t}}$ with $\Xi_{t} \defeq \{A_i\}_{i=1}^{N_{t}} \mysim \pi_{\pparams^{\prime}}(\cdot \vert s)$).
\end{lem}
\begin{proof}
	The proof is similar to arguments in Lemma 7 of \citep{hu2007model}.
	Since state $s$ and expert parameter $\theta$  are considered fixed, we assume the following notation in the proof. Let
	\begin{flalign}
		\widehat{f}^{\rho}_{t \vert s, \theta} \defeq \widehat{f}^{\rho}_{t} \text{ and }f^{\rho}_{\vert s, \theta} \defeq f^{\rho}(Q_{\theta}(s, \cdot), \pparams^{\prime}),
	\end{flalign}
	where $\widehat{f}^{\rho}_{t}$  and $f^{\rho}(Q_{\theta}(s, \cdot), \pparams^{\prime})$ are defined in Equations \eqref{eq:quant_chn} and \eqref{eq:quant_hat}.\\

	Consider the open cover $\{B_{r}(\ell), \ell\in [Q^{\theta}_{l}, Q^{\theta}_{u}]\}$ of $[Q^{\theta}_{l}, Q^{\theta}_{u}]$.
	Since $[Q^{\theta}_{l}, Q^{\theta}_{u}]$ is compact, there exists a finite sub-cover, \emph{i.e.}, $\exists \{\ell_1, \ell_2, \dots, \ell_{M}\}$ \emph{s.t.} $\cup_{i=1}^{M} B_{r}(\ell_i) = [Q^{\theta}_{l}, Q^{\theta}_{u}]$.
	Let $\vartheta(\ell) \defeq \mathbb{E}_{A \sim \pi_{\pparams^{\prime}}(\cdot \vert S)}\left[\Psi(Q_{\theta}(s, A), \ell)\right]$  and
	$\widehat{\vartheta}_{t}(\ell) \defeq \frac{1}{N_t}\sum\limits_{\substack{A \in \Xi_{t}, \vert \Xi_{t} \vert = N_{t},\\ \Xi_{t} \mysim \pi_{\pparams^{\prime}}(\cdot \vert s)}}\Psi(Q_{\theta}(s, A), \ell)$.\\

	Now, by triangle inequality, we have for $\ell \in [Q^{\theta}_{l}, Q^{\theta}_{u}]$,
	\begin{flalign}\label{eq:vt_diif}
		\vert \vartheta(\ell) - \widehat{\vartheta}_{t}(\ell) \vert &\leq \vert \vartheta(\ell) - \vartheta(\ell_{j}) \vert +  \vert \vartheta(\ell_{j}) - \widehat{\vartheta}_{t}(\ell_{j}) \vert +  \vert \widehat{\vartheta}_{t}(\ell_{j}) - \widehat{\vartheta}_{t}(\ell) \vert \nonumber\\
		&\leq L_{\rho}\vert \ell - \ell_{j} \vert + \vert \vartheta(\ell_{j}) - \widehat{\vartheta}_{t}(\ell_{j}) \vert +  \widehat{L}_{\rho}\vert \ell_{j} - \ell \vert \nonumber\\
		&\leq \left(L_{\rho} + \widehat{L}_{\rho}\right)r + \vert \vartheta(\ell_{j}) - \widehat{\vartheta}_{t}(\ell_{j}) \vert,
	\end{flalign}
	where $L_{\rho}$ and $\widehat{L}_{\rho}$ are the Lipschitz constants of $\vartheta(\cdot)$ and $\widehat{\vartheta}_{t}(\cdot)$ respectively.

	For $\delta > 0$, take $r = \delta(L_{\rho} + \widehat{L}_{\rho})/2$. Also, by Kolmogorov's strong law of large numbers (Theorem 2.3.10 of \citep{sen2017large}), we have $\widehat{\vartheta}_{t}(\ell) \rightarrow \vartheta(\ell)$ $a.s.$ This implies that there exists $T \in \mathbb{N}$ \emph{s.t.} $\vert \vartheta(\ell_{j}) - \widehat{\vartheta}_{t}(\ell_{j}) \vert < \delta/2$, $\forall t \geq T$, $\forall j \in [M]$.
	Then from Eq. (\ref{eq:vt_diif}), we have
	\begin{flalign*}
		\vert \vartheta(\ell) - \widehat{\vartheta}_{t}(\ell) \vert \leq \delta/2 + \delta/2 = \delta, \hspace*{4mm} \forall \ell \in [Q^{\theta}_{l}, Q^{\theta}_{u}].
	\end{flalign*}
	This implies $\widehat{\vartheta}_{t}$ converges uniformly to $\vartheta$. By Lemmas \ref{prop:quad_lips} and \ref{prop:quant_convex}, $\widehat{\vartheta}_{t}$ and $\vartheta$ are strictly convex and Lipschitz continuous, and so because $\widehat{\vartheta}_{t}$ converges uniformly to $\vartheta$, this means that the sequence of minimizers of $\widehat{\vartheta}_{t}$ converge to the minimizer of $\vartheta$ (see Lemma \ref{prop:convex_sols_conv}, Appendix \ref{app_lem3} for an explicit justification). These minimizers correspond to $\widehat{f}^{\rho}_{t}$ and $f^{\rho}(Q_{\theta}(s, \cdot), \pparams^{\prime})$ respectively, and so $\lim_{N_t \rightarrow \infty}\widehat{f}^{\rho}_{t} = f^{\rho}(Q_{\theta}(s, \cdot), \pparams^{\prime})$ $a.s.$\\

	Now, for $\delta > 0$ and $r \defeq \delta(L_{\rho} + \widehat{L}_{\rho})/2$, we obtain the following from Eq. (\ref{eq:vt_diif}):
	\begin{flalign*}
		\vert& \vartheta(\ell) - \widehat{\vartheta}_{t}(\ell) \vert \leq \delta/2 + \vert \vartheta(\ell_{j}) - \widehat{\vartheta}_{t}(\ell_{j}) \vert&&\\
		&\Leftrightarrow \{\vert \vartheta(\ell_{j}) - \widehat{\vartheta}_{t}(\ell_{j}) \vert \leq \delta/2, \forall j \in [M]\} \Rightarrow \{\vert \vartheta(\ell) - \widehat{\vartheta}_{t}(\ell) \vert \leq \delta, \forall \ell \in [Q^{\theta}_{l}, Q^{\theta}_{u}]\}
	\end{flalign*}
	\begin{flalign}\label{eq:hoffbound}
		\Rightarrow  \mathbb{P}_{\pi_{\pparams^{\prime}}}\left(\vert \vartheta(\ell) - \widehat{\vartheta}_{t}(\ell) \vert \leq \delta, \forall \ell \in [Q^{\theta}_{l}, Q^{\theta}_{u}]\right) &\geq  \mathbb{P}_{\pi_{\pparams^{\prime}}}\left(\vert \vartheta(\ell_{j}) - \widehat{\vartheta}_{t}(\ell_{j}) \vert \leq \delta/2, \forall j \in [M]\right)\nonumber\\
		&= 1-\mathbb{P}_{\pi_{\pparams^{\prime}}}\left(\vert \vartheta(\ell_{j}) - \widehat{\vartheta}_{t}(\ell_{j}) \vert > \delta/2, \exists  j \in [M]\right)\nonumber\\
		&\geq 1-\sum_{j=1}^{M}\mathbb{P}_{\pi_{\pparams^{\prime}}}\left(\vert \vartheta(\ell_{j}) - \widehat{\vartheta}_{t}(\ell_{j}) \vert > \delta/2\right)\nonumber\\
		&\geq 1-M\max_{j \in [M]}\mathbb{P}_{\pi_{\pparams^{\prime}}}\left(\vert \vartheta(\ell_{j}) - \widehat{\vartheta}_{t}(\ell_{j}) \vert > \delta/2\right)\nonumber\\
		&\geq 1 - 2M \exp{\left(\frac{-2N_{t}\delta^{2}}{4(Q^{\theta}_{u}-Q^{\theta}_{l})^{2}}\right)},
	\end{flalign}
	where $\mathbb{P}_{\pi_{\pparams^{\prime}}} \defeq \mathbb{P}_{A \sim \pi_{\pparams^{\prime}}}(\cdot \vert s)$. And
	the last inequality follows from Hoeffding's inequality \citep{hoeffding1963probability} along with the fact that $\mathbb{E}_{\pi_{\pparams^{\prime}}}\left[\widehat{\vartheta}_{t}(\ell_{j})\right] = \vartheta(\ell_{j})$ and $\sup\limits_{\ell \in [Q^{\theta}_{l}, Q^{\theta}_{u}]} \vert \vartheta(\ell) \vert \leq Q^{\theta}_{u} - Q^{\theta}_{l}$.\\

	Now, the sub-differential of $\vartheta(\ell)$ is given by
	\begin{flalign}
		\partial_{\ell} \vartheta(\ell) = \left[\rho-\mathbb{P}_{A \sim \pi_{\pparams^{\prime}}(\cdot | s)}\left(Q_{\theta}(s, A) \geq \ell\right), \rho-1+\mathbb{P}_{A \sim \pi_{\pparams^{\prime}}(\cdot | s)}\left(Q_{\theta}(s, A) \leq \ell\right)\right].
	\end{flalign}
	By the definition of sub-gradient we obtain
	\begin{flalign}
		&c\vert \widehat{f}^{\rho}_{t \vert s, \theta} - f^{\rho}_{\vert s, \theta} \vert \leq \vert \vartheta(\widehat{f}^{\rho}_{t \vert s, \theta}) - \vartheta(f^{\rho}_{\vert s, \theta}) \vert, \hspace*{2mm} c \in \partial_{\ell}\vartheta(\ell)\nonumber\\
		&\Rightarrow C\vert \widehat{f}^{\rho}_{t \vert s, \theta} - f^{\rho}_{\vert s, \theta} \vert \leq \vert \vartheta(\widehat{f}^{\rho}_{t \vert s, \theta}) - \vartheta(f^{\rho}_{\vert s, \theta}) \vert,
	\end{flalign}
	where $C \defeq \max{\left\{\rho-\mathbb{P}_{A \sim \pi_{\pparams^{\prime}}(\cdot | s)}\left(Q_{\theta}(s, A) \geq f^{\rho}_{\vert s, \theta}\right), \rho-1+\mathbb{P}_{A \sim \pi_{\pparams^{\prime}}(\cdot | s)}\left(Q_{\theta}(s, A) \leq f^{\rho}_{\vert s, \theta}\right)\right\}}$.
	Further,
	\begin{flalign}\label{eq:min_bd}
		C\vert \widehat{f}^{\rho}_{t \vert s, \theta} - f^{\rho}_{\vert s, \theta} \vert &\leq \vert \vartheta(\widehat{f}^{\rho}_{t \vert s, \theta}) - \vartheta(f^{\rho}_{\vert s, \theta}) \vert \nonumber\\
		&\leq  \vert \vartheta(\widehat{f}^{\rho}_{t \vert s, \theta}) - \widehat{\vartheta}_{t}(\widehat{f}^{\rho}_{t \vert s, \theta}) \vert + \vert \widehat{\vartheta}_{t}(\widehat{f}^{\rho}_{t \vert s, \theta}) - \vartheta(f^{\rho}_{\vert s, \theta}) \vert \nonumber\\
		&\leq  \vert \vartheta(\widehat{f}^{\rho}_{t \vert s, \theta}) - \widehat{\vartheta}_{t}(\widehat{f}^{\rho}_{t \vert s, \theta}) \vert + \sup_{\ell \in [Q^{\theta}_{l}, Q^{\theta}_{u}]}\vert \widehat{\vartheta}_{t}(\ell) - \vartheta(\ell) \vert \nonumber\\
		&\leq  2\sup_{\ell \in [Q^{\theta}_{l}, Q^{\theta}_{u}]}\vert \widehat{\vartheta}_{t}(\ell) - \vartheta(\ell) \vert.
	\end{flalign}
	From Eqs. (\ref{eq:hoffbound}) and (\ref{eq:min_bd}), we obtain for $\epsilon > 0$
	\begin{flalign*}
		\mathbb{P}_{\pparams^{\prime}}\left(N_{t}^{\alpha}\vert \widehat{f}^{\rho}_{t \vert s, \theta} - f^{\rho}_{\vert s, \theta} \vert  \geq \epsilon\right) &\leq \mathbb{P}_{\pparams^{\prime}}\left(
		N_{t}^{\alpha}\sup_{\ell \in [Q^{\theta}_{l}, Q^{\theta}_{u}]}\vert \widehat{\vartheta}_{t}(\ell) - \vartheta(\ell) \vert \geq \frac{\epsilon}{2}\right)\\
		&\leq 2M \exp{\left(\frac{-2N_{t}\epsilon^{2}}{16N^{2\alpha}_{t}(Q^{\theta}_{u}-Q^{\theta}_{l})^{2}}\right)}
		= 2M \exp{\left(\frac{-2N_{t}^{1-2\alpha}\epsilon^{2}}{16(Q^{\theta}_{u}-Q^{\theta}_{l})^{2}}\right)}.
	\end{flalign*}
	For $\alpha \in (0,1/2)$ and $\inf_{t \in \mathbb{N}}\frac{N_{t+1}}{N_{t}} \geq \tau > 1$ (by Assumption \ref{assm:assm3}), then
	\begin{flalign*}
		\sum_{t=1}^{\infty}2M \exp{\left(\frac{-2N_{t}^{1-2\alpha}\epsilon^{2}}{16(Q^{\theta}_{u}-Q^{\theta}_{l})^{2}}\right)} \leq
		\sum_{t=1}^{\infty}2M \exp{\left(\frac{-2\tau^{(1-2\alpha)t}N_0^{1-2\alpha}\epsilon^{2}}{16(Q^{\theta}_{u}-Q^{\theta}_{l})^{2}}\right)} < \infty.
	\end{flalign*}
	Therefore, by Borel-Cantelli's Lemma \citep{durrett1991probability}, we have
	\begin{flalign*}
		\mathbb{P}_{\pparams^{\prime}}\left(N_{t}^{\alpha}\big\vert \widehat{f}^{\rho}_{t \vert s, \theta} - f^{\rho}_{\vert s, \theta}
		\big\vert  \geq \epsilon \hspace*{2mm} i.o \right)  = 0.
	\end{flalign*}
	Thus we have $N_{t}^{\alpha}\left(\widehat{f}^{\rho}_{t \vert s, \theta} - f^{\rho}_{\vert s, \theta}\right)$ $\rightarrow$ $0$ $a.s.$ as $N_{t} \rightarrow \infty$.
\end{proof}

\begin{lem}\label{lem:ltheta_conv}
	Almost surely,
	\begin{flalign*}
		\ell^{\theta}_{t} \rightarrow 0 \hspace*{2mm} \text{ as } N_{t} \rightarrow \infty.
	\end{flalign*}
\end{lem}
\begin{proofref}{Proof of Lemma \ref{lem:ltheta_conv}}
	Consider
	\begin{flalign*}
		&\mathbb{E}\bigg[\frac{1}{N_{t}}\sum_{A \in \Xi_{t}}{I_{\{Q_{\theta}(S_{t}, A) \geq \widehat{f}^{\rho}_{t+1}\}}  \nabla_{\pparams_t}  \ln \pi_\pparams(A | S_{t})} \bigg\vert \mathcal{F}_{t}\bigg] = \\
		&\hspace*{5mm}\mathbb{E}\Bigg[\mathbb{E}_{\Xi_{t}}\bigg[\frac{1}{N_{t}}\sum_{A \in \Xi_{t}}{I_{\{Q_{\theta}(S_{t}, A) \geq \widehat{f}^{\rho}_{t+1}\}}  \nabla_{\pparams_t}  \ln \pi_\pparams(A | S_{t})}\bigg] \bigg\vert S_{t}=s, \pparams^{\prime}_{t}\Bigg]
	\end{flalign*}
Notice that, because of the conditions on $\pi_{\pparams^{\prime}}(\cdot \vert s)$, we know that the sample average converges with an exponential rate in the number of samples, for arbitrary $\pparams^{\prime} \in W$.
Namely, for $\epsilon > 0$ and $N \in \mathbb{N}$, we have
\begin{flalign*}
	&\mathbb{P}_{\Xi \mysim \pi_{\pparams^{\prime}}(\cdot \vert s)}\Big(\Big\Vert \frac{1}{N}\sum_{A \in \Xi}{I_{\{Q_{\theta}(s, A) \geq f^{\rho}(Q_{\theta}(s, \cdot), \pi_{\pparams^{\prime}}(\cdot \vert s)\}}  \nabla_{\pparams}  \ln \pi_\pparams(A | s)} - \\
	&\hspace*{5mm}\mathbb{E}_{A \sim \pi_{\pparams^{\prime}}(\cdot \vert s)}\left[I_{\{Q_{\theta}(s, A) \geq \widehat{f}^{\rho}(Q_{\theta}(s, \cdot), \pi_{\pparams^{\prime}}(\cdot \vert s)\}}  \nabla_{\pparams}  \ln \pi_\pparams(A | s)\right] \Big\Vert \geq \epsilon\Big)
		\leq C_{1}\exp{\left(-c_2N^{c_3}\epsilon^{c_4}\right)},\\[4pt]
	&\hspace*{70mm}\forall \theta \in \Theta, \pparams, \pparams^{\prime} \in  W, s \in \States,
\end{flalign*}
where $C_1, c_2, c_3, c_4 > 0$.

Therefore,	 for $\alpha^{\prime} > 0$, we have
	\begin{flalign*}
		&\mathbb{P}\Big(N_{t}^{\alpha^{\prime}}\Big\Vert \frac{1}{N_{t}}\sum_{A \in \Xi_{t}}{I_{\{Q_{\theta}(s, A) \geq \widehat{f}^{\rho}_{\theta, s}\}}  \nabla_{\pparams_t}  \ln \pi_\pparams(A | s)} - \mathbb{E}\left[I_{\{Q_{\theta}(s, A) \geq \widehat{f}^{\rho}_{\theta, s}\}}  \nabla_{\pparams_t}  \ln \pi_\pparams(A | s)\right] \Big\Vert \geq \epsilon\Big)\\
		&\hspace*{15mm}\leq C_{1}\exp{\left(-\frac{c_2N^{c_3}_{t}\epsilon^{c_4}}{N_{t}^{c_4\alpha^{\prime}}}\right)}
		= C_{1}\exp{\left(-c_2N^{c_3 - c_4\alpha^{\prime}}_{t}\epsilon^{c_4}\right)}\\
		&\hspace*{15mm}\leq C_{1}\exp{\left(-c_2\tau^{(c_3 - c_4\alpha^{\prime})t}N^{c_3 - c_4\alpha^{\prime}}_{0}\epsilon^{c_4}\right)},
	\end{flalign*}
	where $f^{\rho}_{\theta, s} \defeq f^{\rho}(Q_{\theta}(s, \cdot), \pi_{\pparams^{\prime}}(\cdot \vert s))$ and $\inf_{t \in \mathbb{N}}\frac{N_{t+1}}{N_{t}} \geq \tau > 1$ (by Assumption \ref{assm:assm3}).\\

	For $c_3 - c_4\alpha^{\prime} > 0$ $\Rightarrow$ $\alpha^{\prime} < c_3/c_4$, we have
	\begin{flalign*}
		\sum_{t=1}^{\infty}{C_{1}\exp{\left(-c_2\tau^{(c_3 - c_4\alpha^{\prime})t}N^{c_3 - c_4\alpha^{\prime}}_{0}\epsilon^{c_4}\right)}} < \infty.
	\end{flalign*}
	Therefore, by Borel-Cantelli's Lemma \citep{durrett1991probability}, we have
	\begin{flalign*}
		&\mathbb{P}\Big(N_{t}^{\alpha^{\prime}}\Big\Vert \frac{1}{N_{t}}\sum_{A \in \Xi_{t}}{I_{\{Q_{\theta}(s, A) \geq \widehat{f}^{\rho}_{\theta, s}\}}  \nabla_{\pparams_t}  \ln \pi_\pparams(A | s)} - \mathbb{E}\left[I_{\{Q_{\theta}(s, A) \geq \widehat{f}^{\rho}_{\theta, s}\}}  \nabla_{\pparams_t}  \ln \pi_\pparams(A | s)\right] \Big\Vert \geq \epsilon \hspace*{2mm}i.o.\Big) \\
		&\hspace*{20mm}= 0.
	\end{flalign*}
	This implies that
	\begin{flalign}\label{eq:sl_conv}
		&N_{t}^{\alpha^{\prime}}\Big\Vert \frac{1}{N_{t}}\sum_{A \in \Xi_{t}}{I_{\{Q_{\theta}(s, A) \geq \widehat{f}^{\rho}_{\theta, s}\}}  \nabla_{\pparams_t}  \ln \pi_\pparams(A | s)} - \mathbb{E}\left[I_{\{Q_{\theta}(s, A) \geq \widehat{f}^{\rho}_{\theta, s}\}}  \nabla_{\pparams_t}  \ln \pi_\pparams(A | s)\right] \Big\Vert \rightarrow 0 \hspace*{4mm} a.s.
	\end{flalign}
	The above result implies that the sample average converges at a rate $O(N_{t}^{\alpha^{\prime}})$, where $0 < \alpha^{\prime} < c_3/c_4$ independent of $\pparams, \pparams^{\prime} \in W$. By Lemma \ref{lem:quant_conv}, we have the sample quantiles $\widehat{f}^{\rho}_{t}$  also converging to the true quantile at a rate $O(N_{t}^{\alpha})$ independent of $\pparams, \pparams^{\prime} \in W$.
	Now the claim follows directly from Assumption \ref{assm:assm6} (ii) and bounded convergence theorem.
\end{proofref}

\subsection{Supporting Lemmas for Lemma \ref{lem:quant_conv}}
\begin{lem}\label{prop:quad_lips}
	Let Assumption \ref{assm:assm5} hold.
	For $\theta \in \Theta$, $\pparams^{\prime} \in W$, $s \in \States$ and $\ell \in [Q^{\theta}_{l}, Q^{\theta}_{u}]$, we have
	\begin{enumerate}
	\item $\mathbb{E}_{A \sim \pi_{\pparams^{\prime}}(\cdot \vert s)}\left[\Psi(Q_{\theta}(s, A), \ell)\right]$ is Lipschitz continuous.
	\item
	$\frac{1}{N}\sum_{\substack{A \in \Xi \\ \vert \Xi \vert = N}}\Psi(Q_{\theta}(s, A), \ell)$  (with $\Xi \mysim \pi_{\pparams^{\prime}}(\cdot | s)$) is Lipschitz continuous with Lipschitz constant independent of the sample length $N$.
	\end{enumerate}
\end{lem}
\begin{proof}
	Let $\ell_1, \ell_2 \in [Q^{\theta}_{l}, Q^{\theta}_{u}]$, $\ell_2 \geq \ell_1$. By Assumption \ref{assm:assm5} we have $\mathbb{P}_{A \sim \pi_{\pparams^{\prime}}(\cdot \vert s)}(Q_{\theta}(s, A) \geq \ell_1) > 0$ and $\mathbb{P}_{A \sim \pi_{\pparams^{\prime}}(\cdot \vert s)}(Q_{\theta}(s, A) \geq \ell_2) > 0$. Now,

	\begin{align*}
		\Big\vert &\mathbb{E}_{A \sim \pi_{\pparams^{\prime}}(\cdot \vert s)}
		\left[\Psi(Q_{\theta}(s, A), \ell_1)\right] -  \mathbb{E}_{A \sim \pi_{\pparams^{\prime}}(\cdot \vert s)}\left[\Psi(Q_{\theta}(s, A), \ell_2)\right] \Big\vert \\
		&= \Big\vert \mathbb{E}_{A \sim \pi_{\pparams^{\prime}}(\cdot \vert s)}\left[(Q_{\theta}(s, A) - \ell_1)(1-\rho) I_{\{Q_{\theta}(s, A) \geq \ell_1 \}} +  (\ell_1 - Q_{\theta}(s, A))\rho I_{\{\ell_1 \geq Q_{\theta}(s, A)\}}\right] \\
		&\hspace*{4mm}- \mathbb{E}_{A \sim \pi_{\pparams^{\prime}}(\cdot \vert s)}\left[(Q_{\theta}(s, A) - \ell_2)(1-\rho) I_{\{Q_{\theta}(s, A) \geq \ell_2 \}} +  (\ell_2 - Q_{\theta}(s, A))\rho I_{\{\ell_2 \geq Q_{\theta}(s, A)\}}\right] \Big\vert\\
		&= \Big\vert \mathbb{E}_{A \sim \pi_{\pparams^{\prime}}(\cdot \vert s)}\Big[(Q_{\theta}(s, A) - \ell_1)(1-\rho) I_{\{Q_{\theta}(s, A) \geq \ell_1 \}} +  (\ell_1 - Q_{\theta}(s, A))\rho I_{\{\ell_1 \geq Q_{\theta}(s, A)\}} \\
		&\hspace*{4mm}- (Q_{\theta}(s, A) - \ell_2)(1-\rho) I_{\{Q_{\theta}(s, A) \geq \ell_2 \}} +  (\ell_2 - Q_{\theta}(s, A))\rho I_{\{\ell_2 \geq Q_{\theta}(s, A)\}}\Big] \Big\vert\\
		&= \Big\vert \mathbb{E}_{A \sim \pi_{\pparams^{\prime}}(\cdot \vert s)}\Big[(1-\rho)(\ell_2 - \ell_1)I_{\{Q_{\theta}(s, A) \geq \ell_2 \}} +  \rho(\ell_1 - \ell_2) I_{\{Q_{\theta}(s, A) \leq \ell_1\}} +\\
		&\hspace*{4mm}+ \left(-(1-\rho)\ell_1 - \rho\ell_2 + \rho Q_{\theta}(s, A) + (1-\rho)Q_{\theta}(s, A)\right)I_{\{\ell_1 \leq Q_{\theta}(s, A) \leq \ell_2\}}\Big]\Big\vert\\
		&\leq (1-\rho)\vert\ell_2 - \ell_1\vert + \left(2\rho+1\right)\vert \ell_2 - \ell_1 \vert\\
		&=(\rho + 2)\vert\ell_2 - \ell_1\vert.
	\end{align*}
	Similarly, we can prove the later claim also. This completes the proof of Lemma \ref{prop:quad_lips}.
\end{proof}

\begin{lem}\label{prop:quant_convex}
	Let Assumption \ref{assm:assm5} hold. Then, for $\theta \in \Theta$, $\pparams^{\prime} \in W$, $s \in \States$ and $\ell \in [Q^{\theta}_{l}, Q^{\theta}_{u}]$, we have $\mathbb{E}_{A \sim \pi_{\pparams^{\prime}}(\cdot \vert s)}\left[\Psi(Q_{\theta}(s, A), \ell)\right]$ and
$\frac{1}{N}\sum_{\substack{A \in \Xi \\ \vert \Xi \vert = N}}\Psi(Q_{\theta}(s, A), \ell)$  (with $\Xi \mysim \pi_{\pparams^{\prime}}(\cdot | s)$) are strictly convex.
\end{lem}
\begin{proof}
For $\lambda \in [0,1]$ and $\ell_1, \ell_2 \in [Q_{l}, Q_{u}]$ with $\ell_1 \leq \ell_2$, we have
\begin{align}
	&\mathbb{E}_{A \in \pi_{\pparams^{\prime}}(\cdot \vert S)}\big[\Psi(Q_{\theta}(S,A), \lambda\ell_1 + (1-\lambda)\ell_2)\big] \label{eq_quant_main}\\
	&=\mathbb{E}_{A \in \pi_{\pparams^{\prime}}(\cdot \vert S)}\big[(1-\rho)\big(Q_{\theta}(S,A) - \lambda\ell_1 - (1-\lambda)\ell_2\big)I_{\{Q_{\theta}(S,A) \geq \lambda\ell_1 + (1-\lambda)\ell_2\}} \nonumber\\
	& \hspace{2.0cm} +  \rho\big(\lambda\ell_1 + (1-\lambda)\ell_2 - Q_{\theta}(S,A)\big)I_{\{Q_{\theta}(S,A) \leq \lambda\ell_1 + (1-\lambda)\ell_2\}}\big]\nonumber
	.
\end{align}
Notice that
\begin{align*}
	&\big(Q_{\theta}(S,A) - \lambda\ell_1 - (1-\lambda)\ell_2\big)I_{\{Q_{\theta}(S,A) \geq \lambda\ell_1 + (1-\lambda)\ell_2\}}\\
	&=\big(\lambda Q_{\theta}(S,A) - \lambda\ell_1 +  (1-\lambda) Q_{\theta}(S,A) - (1-\lambda)\ell_2\big)I_{\{Q_{\theta}(S,A) \geq \lambda\ell_1 + (1-\lambda)\ell_2\}}
\end{align*}
We consider how one of these components simplifies.
\begin{align*}
	&\mathbb{E}_{A \in \pi_{\pparams^{\prime}}(\cdot \vert S)}\big[\big(\lambda Q_{\theta}(S,A) - \lambda\ell_1\big)I_{\{Q_{\theta}(S,A) \geq \lambda\ell_1 + (1-\lambda)\ell_2\}}\big]\\
	&= \lambda \mathbb{E}_{A \in \pi_{\pparams^{\prime}}(\cdot \vert S)}\big[\big(Q_{\theta}(S,A) - \ell_1\big)I_{\{Q_{\theta}(S,A) \geq \lambda\ell_1\}} - \big(Q_{\theta}(S,A) - \ell_1\big) I_{\lambda \ell_1 \le \{Q_{\theta}(S,A) \leq \lambda\ell_1 + (1-\lambda)\ell_2\}} \big]\\
	&\le \lambda \mathbb{E}_{A \in \pi_{\pparams^{\prime}}(\cdot \vert S)}\big[\big(Q_{\theta}(S,A) - \ell_1\big) I_{\{Q_{\theta}(S,A) \geq \lambda\ell_1\}} \big] \ \ \ \ \triangleright \ - \big(Q_{\theta}(S,A) - \ell_1\big) \le 0 \\
		& \hspace{8.0cm} \text{ for } \lambda \ell_1 \le \{Q_{\theta}(S,A) \leq \lambda\ell_1 + (1-\lambda)\ell_2\}\\
	&\le \lambda \mathbb{E}_{A \in \pi_{\pparams^{\prime}}(\cdot \vert S)}\big[\big(Q_{\theta}(S,A) - \ell_1\big) I_{\{Q_{\theta}(S,A) \geq \ell_1\}} \big] \ \ \ \ \triangleright \ \big(Q_{\theta}(S,A) - \ell_1\big) \le 0 \text{ for } I_{\lambda\ell_1 \le \{Q_{\theta}(S,A) \le \ell_1\}}
\end{align*}
Similarly, we get
\begin{align*}
	\mathbb{E}_{A \in \pi_{\pparams^{\prime}}(\cdot \vert S)}&\big[\big(Q_{\theta}(S,A) - \ell_2\big)I_{\{Q_{\theta}(S,A) \geq \lambda\ell_1 + (1-\lambda)\ell_2\}}\big]	\le  \mathbb{E}_{A \in \pi_{\pparams^{\prime}}(\cdot \vert S)}\big[\big(Q_{\theta}(S,A) - \ell_2\big) I_{\{Q_{\theta}(S,A) \geq \ell_2\}} \big] \\
	\mathbb{E}_{A \in \pi_{\pparams^{\prime}}(\cdot \vert S)}&\big[\big(\ell_1 - Q_{\theta}(S,A) \big)I_{\{Q_{\theta}(S,A) \leq \lambda\ell_1 + (1-\lambda)\ell_2\}}\big]	\le  \mathbb{E}_{A \in \pi_{\pparams^{\prime}}(\cdot \vert S)}\big[\big(\ell_1 - Q_{\theta}(S,A)\big) I_{\{Q_{\theta}(S,A) \leq \ell_1\}} \big] \\
	\mathbb{E}_{A \in \pi_{\pparams^{\prime}}(\cdot \vert S)}&\big[\big(\ell_2 - Q_{\theta}(S,A) \big)I_{\{Q_{\theta}(S,A) \leq \lambda\ell_1 + (1-\lambda)\ell_2\}}\big]	\le  \mathbb{E}_{A \in \pi_{\pparams^{\prime}}(\cdot \vert S)}\big[\big(\ell_2 - Q_{\theta}(S,A)\big) I_{\{Q_{\theta}(S,A) \leq \ell_2\}} \big]
\end{align*}
Therefore, for Equation \eqref{eq_quant_main}, we get
\begin{align*}
\eqref{eq_quant_main} &\le
	\lambda (1-\rho)\mathbb{E}_{A \in \pi_{\pparams^{\prime}}(\cdot \vert S)}\big[\big(Q_{\theta}(S,A) - \ell_1\big) I_{\{Q_{\theta}(S,A) \geq \ell_1\}} \big] \\
	&\ \ \ + (1-\lambda) (1-\rho)\mathbb{E}_{A \in \pi_{\pparams^{\prime}}(\cdot \vert S)}\big[\big(Q_{\theta}(S,A) - \ell_2\big) I_{\{Q_{\theta}(S,A) \geq \ell_2\}} \big] \\
	&\ \ \ + \lambda \rho\mathbb{E}_{A \in \pi_{\pparams^{\prime}}(\cdot \vert S)}\big[\big(\ell_1 - Q_{\theta}(S,A)\big) I_{\{Q_{\theta}(S,A) \leq \ell_1\}} \big] \\
	& \ \ \ + (1-\lambda) \rho \mathbb{E}_{A \in \pi_{\pparams^{\prime}}(\cdot \vert S)}\big[\big(\ell_2 - Q_{\theta}(S,A)\big) I_{\{Q_{\theta}(S,A) \leq \ell_2\}} \big] \\
	&= \lambda\mathbb{E}_{A \in \pi_{\pparams^{\prime}}(\cdot \vert S)}\left[\Psi(Q_{\theta}(S,A), \ell_1)\right] + (1-\lambda)\mathbb{E}_{A \in \pi_{\pparams^{\prime}}(\cdot \vert S)}\left[\Psi(Q_{\theta}(S,A), \ell_2)\right].
\end{align*}
We can prove the second claim similarly. This completes the proof of Lemma \ref{prop:quant_convex}.
\end{proof}

\subsection{Lemma \ref{prop:convex_sols_conv}}\label{app_lem3}

\begin{lem}\label{prop:convex_sols_conv}
	Let $\{f_{n} \in C(\bbbr, \bbbr)\}_{n \in \mathbb{N}}$ be a sequence of strictly convex, continuous functions converging uniformly to a strict convex function $f$. Let $x_{n}^{*} = \argmin_{x}{f_{n}(x)}$ and $x^{*} = \argmin_{x \in \bbbr}{f(x)}$. Then $\lim\limits_{n \rightarrow \infty}x^{*}_{n} = x^{*}$.
\end{lem}
\begin{proof}
	Let $c = \liminf_{n} x^{*}_{n}$. We employ proof by contradiction here. For that, we assume $x^{*} > c$. Now, note that $f(x^{*}) < f(c)$ and $f(x^{*}) < f(\left(x^{*}+c\right)/2)$ (by the definition of $x^{*}$). Also, by the strict convexity of $f$, we have $f((x^{*} + c)/2) < \left(f(x^{*}) + f(c)\right)/2$ $ < f(c)$.  Therefore, we have
\begin{flalign}
	f(c) > f((x^{*} + c)/2) > f(x^{*}).
\end{flalign}
	Let $r_1 \in \bbbr$ be such that $f(c) > r_1 >  f((x^{*} + c)/2)$.
	Now, since $\Vert f_{n} - f^{*} \Vert_{\infty}$ $\rightarrow 0$ as $n \rightarrow \infty$, there exists an positive integer $N$ \emph{s.t.} $|f_{n}(c) - f(c)| < f(c)-r_1$, $\forall n \geq N$ and $\epsilon > 0$. Therefore,
	$f_{n}(c) - f(c) > r_1 - f(c)$ $\Rightarrow$ $f_{n}(c) > r_1$. Similarily, we can show that $f_{n}((x^{*} + c)/2) > r_1$. Therefore, we have $f_{n}(c) > f_{n}((x^{*} + c)/2)$.
	Similarily, we can show that $f_{n}((x^{*} + c)/2) > f_{n}(x^{*})$. Finally, we obtain
	\begin{flalign}\label{eq:bdfn}
		f_{n}(c) > f_{n}((x^{*} + c)/2) > f_{n}(x^{*}), \hspace*{4mm} \forall n \geq N.
	\end{flalign}
	Now, by the extreme value theorem of the continuous functions, we obtain that for $n \geq N$, $f_n$ achieves minimum (say at $x_{p}$ in the closed interval $[c, (x^{*} + c)/2]$. Note that $f_{n}(x_{p}) \nless f_{n}((x^{*} + c)/2)$ (if so then $f_{n}(x_{p})$ will be a local minimum of $f_{n}$ since $f_{n}(x^{*}) < f_{n}((x^{*} + c)/2)$). Also, $f_{n}(x_{p}) \neq f_{n}((x^{*}+c)/2)$. Therefore, $f_{n}$ achieves it minimum in the closed interval $[c, (x^{*} + c)/2]$ at the point $(x^{*} + c)/2$. This further implies that $x^{*}_{n} > (x^{*} + c)/2$. Therefore, $\liminf_{n} x_{n}^{*} \geq (x^{*} + c)/2$ $\Rightarrow$ $c \geq (x^{*} + c)/2$ $\Rightarrow$ $c \geq x^{*}$. This is a contradiction and implies
\begin{equation}\label{eq:liminfres}
	\liminf_{n} x^{*}_{n} \geq x^{*}.
\end{equation}
	Now  consider $g_{n}(x) = f_{n}(-x)$. Note that $g_{n}$ is also continuous and strictly convex. Indeed, for $\lambda \in [0,1]$, we have $g_{n}(\lambda x_1 + (1-\lambda)x_2) = f_{n}(-\lambda x_1 - (1-\lambda)x_2) < \lambda f(-x_1) + (1-\lambda)f(-x_2) = \lambda g(x_1) + (1-\lambda)g(x_2)$. Applying the result from Eq. (\ref{eq:liminfres}) to the sequence $\{g_{n}\}_{n \in \mathbb{N}}$, we obtain that $\liminf_{n}(-x_{n}^{*}) \geq -x^{*}$. This implies $\limsup_{n} x_{n}^{*} \leq x^{*}$. Therefore,
\begin{flalign*}
	\liminf_{n} x^{*}_{n} \geq x^{*} \geq \limsup_{n} x_{n}^{*} \geq \limsup_{n} x_{n}^{*}.
\end{flalign*}
	Hence, $\liminf_{n} x^{*}_{n} = \limsup_{n} x_{n}^{*} = x^{*}$
\end{proof}


\section{Experimental Details}

\subsection{Hyperparameter Details}%
\label{apdx:hypers}

In this section, we outline the tuned hyperparameters for each algorithm on each environment in our experiments.
For each algorithm, hyperparameters were tuned over an initial 10 runs with different
random seeds. Each algorithm saw the same 10 initial random seeds. For a list of all hyperparameters swept, see
Section~\ref{sec:exp_details}.
In Table~\ref{tab:continuous_across_env}, we list the tuned hyperparameters for each algorithm when
tuning across continuous-action environments.
In Table~\ref{tab:discrete_across_env}, we list the tuned hyperparameters for each algorithm when
tuning across discrete-action environments.
In Tables~\ref{tab:per_env_greedyac}, \ref{tab:per_env_vanillaac}, and \ref{tab:per_env_sac}, we list the tuned
hyperparameters when tuning per-environment for GreedyAC, VanillaAC, and SAC respectively. Finally,
Table~\ref{tab:swim_greedyac} outlines the hyperparamters used in the experiments on Swimmer.

\begin{table}[H]
\centering
\begin{tabular}{| c | c | c | c |}
	\hline
	Hyperparameter & $\kappa$ & $\alpha$ & $\tau$ \\
	\hline
	Greedy Actor-Critic & 1.0 & 1e-3 & 1e-3 \\
	Vanilla Actor-Critic & 2.0 & 1e-3 & 1e-3 \\
	Soft Actor-Critic & 1.0 & 1e-3 & 1e-3 \\
	\hline
\end{tabular}
	\caption{Hyperparameters tuned across continuous-action
	environments for GreedyAC, VanillaAC, and SAC.}
\label{tab:continuous_across_env}
\end{table}

\begin{table}[H]
\centering
\begin{tabular}{| c | c | c | c |}
	\hline
	Hyperparameter & $\kappa$ & $\alpha$ & $\tau$ \\
	\hline
	Greedy Actor-Critic & 10.0 & 1e-3 & - \\
	Vanilla Actor-Critic & 1e-1 & 1e-3 & 1e-2 \\
	Soft Actor-Critic & 10 & 1e-5 & 10 \\
	\hline
\end{tabular}
\caption{Hyperparameters tuned across discrete-action
	environments for GreedyAC, VanillaAC, and SAC.}
\label{tab:discrete_across_env}
\end{table}

\begin{table}[H]
\centering
\begin{tabular}{| c | c | c | c |}
	\hline
	Hyperparameter & $\kappa$ & $\alpha$ & $\tau$ \\
	\hline
	Acrobot-CA & 1e-1 & 1e-3 & 1e-2 \\
	Acrobot-DA & 1e-1 & 1e-2 & - \\
	Mountain Car-CA & 1.0 & 1e-3 & 10.0 \\
	Mountain Car-DA & 2.0 & 1e-3 & - \\
	Pendulum-CA & 1e-1 & 1e-2 & 10.0 \\
	Pendulum-DA & 1.0 & 1e-3 & - \\
	\hline
\end{tabular}
\caption[Per-Environment Tuned Hyperparameters for GreedyAC]{Hyperparameters tuned per-environment for GreedyAC.}
\label{tab:per_env_greedyac}
\end{table}

\begin{table}[H]
\centering
\begin{tabular}{| c | c | c | c |}
	\hline
	Hyperparameter & $\kappa$ & $\alpha$ & $\tau$ \\
	\hline
	Acrobot-CA & 2.0 & 1e-3 & 1e-3 \\
	Acrobot-DA & 1e-1 & 1e-2 & 1e-2 \\
	Mountain Car-CA & 2.0 & 1e-3 & 1e-3 \\
	Mountain Car-DA & 1.0 & 1e-3 & 1e-2 \\
	Pendulum-CA & 1e-2 & 1e-2 & 1e-2 \\
	Pendulum-DA & 2.0 & 1e-3 & 1.0 \\
	\hline
\end{tabular}
\caption[Per-Environment Tuned Hyperparameters for VanillaAC]{Hyperparameters tuned per-environment for VanillaAC.}
\label{tab:per_env_vanillaac}
\end{table}

\begin{table}[H]
\centering
\begin{tabular}{| c | c | c | c |}
	\hline
	Hyperparameter & $\kappa$ & $\alpha$ & $\tau$ \\
	\hline
	Acrobot-CA & 10.0 & 1e-5 & 10.0 \\
	Acrobot-DA & 2.0 & 1e-5 & 10.0 \\
	Mountain Car-CA & 1.0 & 1e-3 & 1e-3 \\
	Mountain Car-DA & 1.0 & 1e-3 & 1e-2 \\
	Pendulum-CA & 1e-1 & 1e-2 & 1e-1 \\
	Pendulum-DA & 1.0 & 1e-3 & 1.0 \\
	\hline
\end{tabular}
\caption[Per-Environment Tuned Hyperparameters for SAC]{Hyperparameters tuned per-environment for SAC.}
\label{tab:per_env_sac}
\end{table}

\begin{table}[H]
\centering
\begin{tabular}{| c | c | c | c |}
	\hline
	Hyperparameter & $\kappa$ & $\alpha$ & $\tau$ \\
	\hline
	Greedy Actor-Critic & 1e-2 & 1e-4 & 1e-1 \\
	\hline
\end{tabular}
	\caption{Hyperparameters Chosen for GreedyAC on Swimmer.}
\label{tab:swim_greedyac}
\end{table}

\subsection{Normalization Approach}\label{app_normalization}

For each environment, we find the best return achieved by any agent, across all runs, as a simple approximation to a
near-optimal return. Table~\ref{tab:NearOptimalPerformance} lists these returns for each environment. Then, to obtain a
normalized score, we use $1 - \frac{\text{BestValue} - \text{AlgValue}}{|\text{BestValue}|}$, where the numerator is
guaranteed to be nonnegative. If $\text{AlgValue} = \text{BestValue}$ we get the highest value of 1. If AlgValue is half
of BestValue, we get $\frac{0.5\text{BestValue}}{|\text{BestValue}|} = 0.5$. If AlgValue is significantly worse than
BestValue, the score is much lower.

The AlgValue that we normalize is the point depicted on the sensitivity plot. It corresponds to the Average Return
across timesteps and across runs for the algorithm, with that hyperparameter setting in that environment.

For the experiments in Figure~\ref{fig:across-env}, where we tune across the complete set of
discrete- or continuous-action environments,
we first compute the normalized scores just
described. Then, we compute the average normalized scores for each algorithm and hyperparameter setting across
discrete-action and continuous-action environments separately. We then choose the hyperparameter setting for each
algorithm for the discrete- and continuous-action environments based on the hyperparameter setting which resulted in the
highest normalized scores. The learning curves for these hyperparameters, for each algorithm, are shown in
Figure~\ref{fig:across-env}.

\begin{table}[H]
\centering
\begin{tabular}{ | c | c | c | }
	\hline
	Environment & Continuous & Discrete \\
	\hline
	Acrobot & -56 & -56 \\
	Mountain Car & -65 & -83 \\
	Pendulum & 930 & 932 \\
	\hline
\end{tabular}
\caption{Approximate return achieved by an optimal policy. We approximate the return achievable by a near-optimal policy
on environment $\mathscr{E}$ by finding the highest return achieved over all runs of all hyperparameters and all agents
on environment $\mathscr{E}$. }
\label{tab:NearOptimalPerformance}
\end{table}

\subsection{Sensitivity Plots}\label{app_sensitivity}

We plot parameter sensitivity curves, which include a line for each entropy scale, with the stepsize on the x-axis.
Because there are two stepsizes, we have two sets of plots -- one for the critic stepsize and one for the actor stepsize.
When examining the sensitivity to the critic stepsize, we select the corresponding best actor stepsize. This means that
for each point $(\text{critic stepsize}, \text{entropy scale}) = (\alpha, \tau)$ on the sensitivity plot for the critic
stepsize, we find the best actor stepsize and report
the performance for that triplet averaged over all 40 runs.
We do the same procedure when plotting the actor stepsize on the x-axis, but maximizing over critic stepsize.

\begin{figure}[ht]
	\centering
	\includegraphics[width=0.9\linewidth]{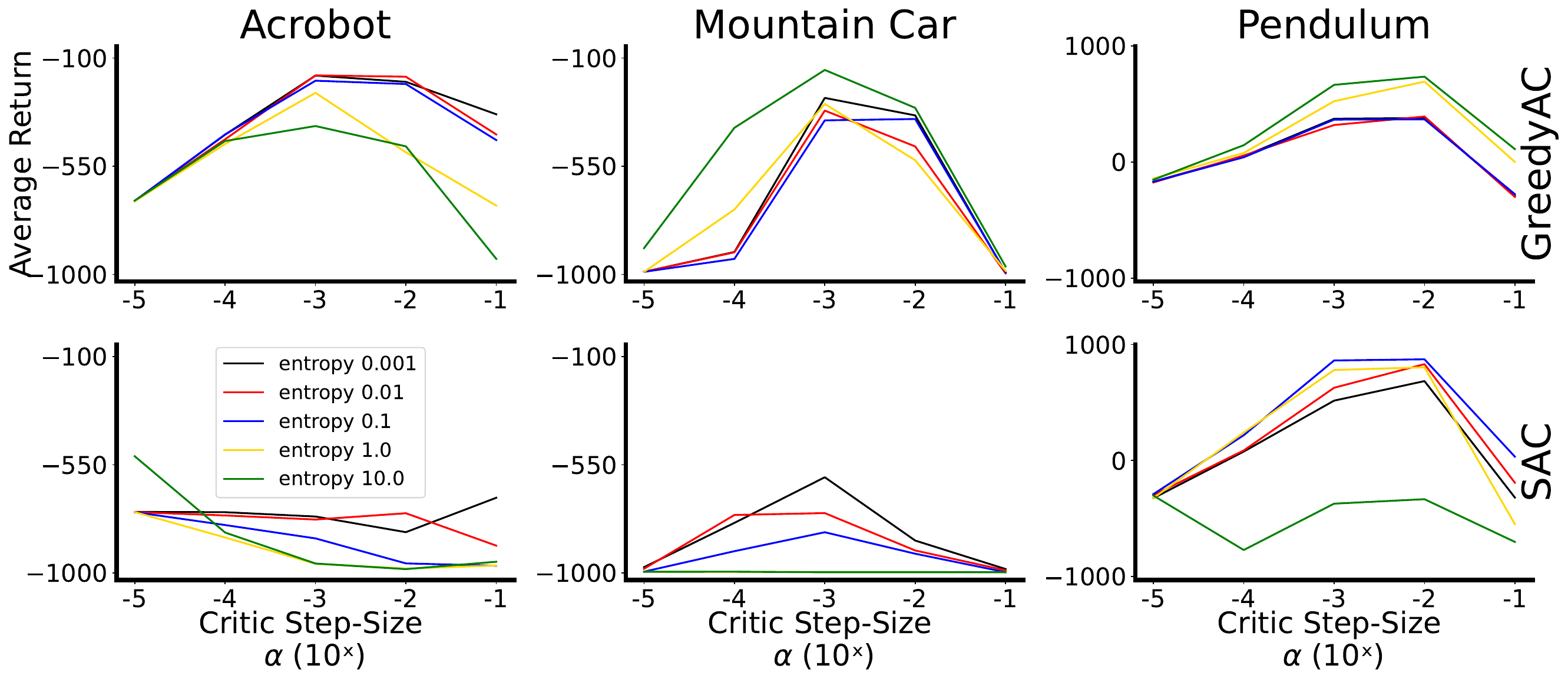}
	\caption{Sensitivity curves for the critic step-size hyperparameter $\alpha$ for GreedyAC and SAC,
	with one line for each entropy scale tested. The critic step-size
	is plotted on a logarithmic scale on the x-axis.}
	\label{fig:entropyscale-critic}
\end{figure}

\begin{figure}[H]
	\centering
	\includegraphics[width=0.95\linewidth]{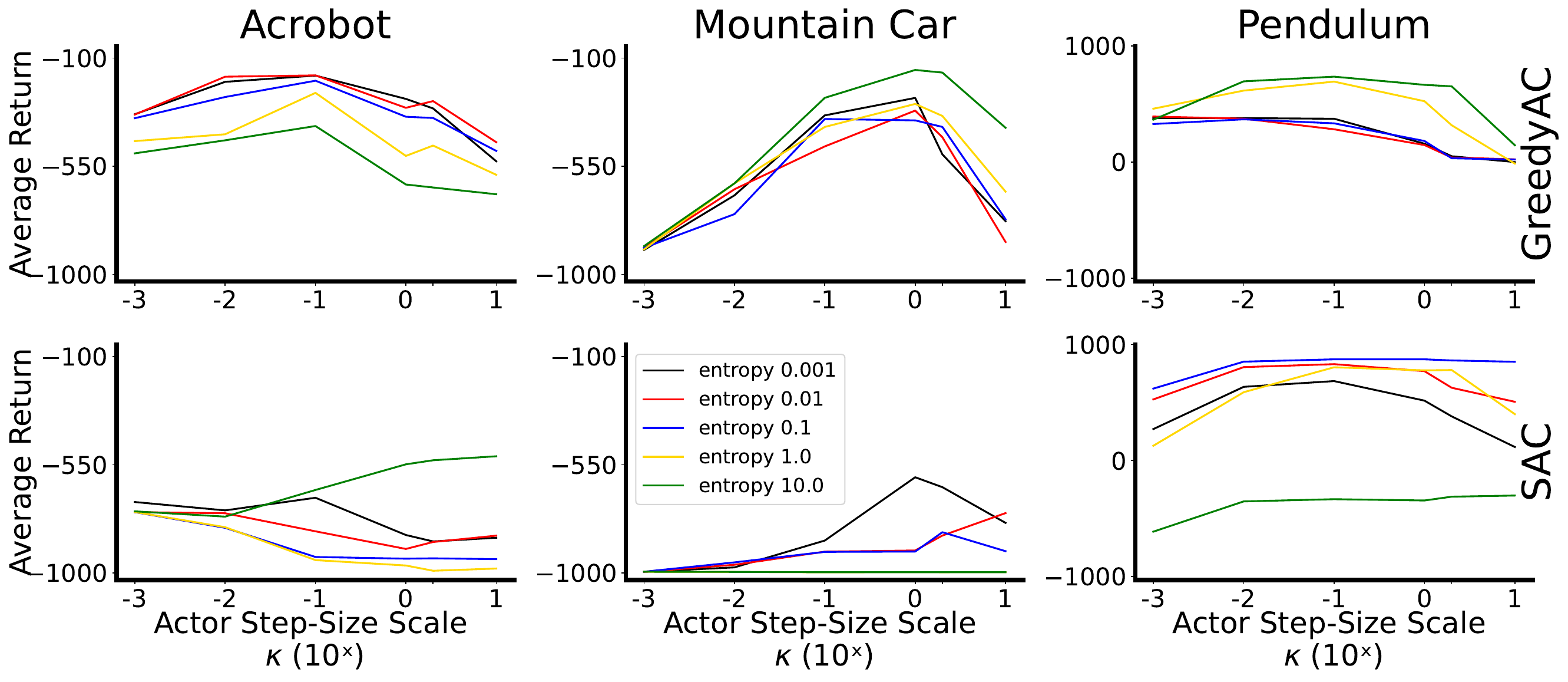}
	\caption{Sensitivity curves for the actor stepsize scale hyperparameter $\kappa$
	for both GreedyAC and SAC, with one line for each entropy scale tested. The actor step-size scale is plotted on a
	logarithmic scale on the x-axis.}
	\label{fig:entropyscale-actor}
\end{figure}

\section{Ablation Study on SAC}%
\label{apdx:sac_ablation}

\begin{figure}[H]
	\centering
	\includegraphics[width=1.0\linewidth]{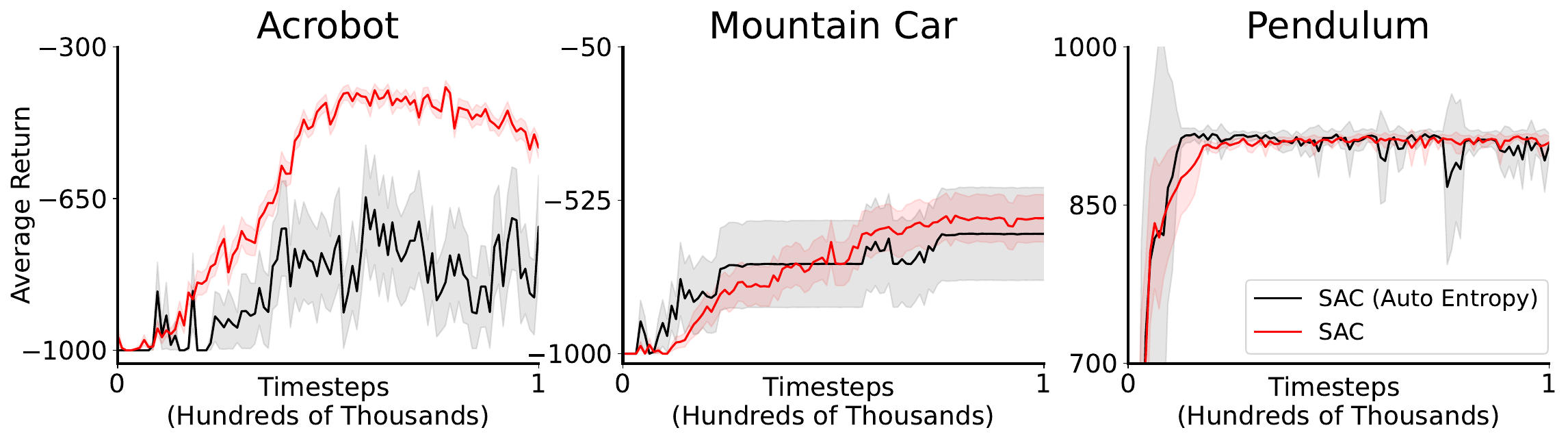}
	\caption{Learning curves over 40 runs for SAC and 10 runs for SAC (Auto Entropy) with shaded regions denoting
	standard error. The entropy scale for SAC as well as the entropy scale step-size for SAC (Auto Entropy)
	are \textbf{tuned using a grid search}.}
	\label{fig:auto-ent}
\end{figure}

Modern variants of SAC utilize a trick to automatically adapt the entropy scale hyperparameter during training
\citep{SACv2}. In order to gauge which variant of SAC to use in this work, we performed an ablation study where we
studied SAC with and without automatic entropy tuning. We ran SAC with automatic entropy tuning for 10 runs.
Hyperparameters were swept in the same sets as listed in Section~\ref{sec:exp_details}.
Additionally, we swept entropy scale step-sizes $\beta = 10^{z}$ for in $z \in
\left\{ -4, -3, -2 \right\}$ for automatic entropy tuning.
Figure~\ref{fig:auto-ent} shows the learning curves of SAC with automatic entropy tuning,
over 10 runs, and SAC without automatic entropy tuning over the 40 runs conducted for the experiments in the main text.
As can be seen in the figure, performing a grid search over the entropy scale hyperparameter never degrades performance
compared to using automatic entropy tuning, and in some cases results in better performance than when
using automatic entropy tuning. Because of this, we decided to use manual entropy tuning through a grid search in our
experiments, which also allows us to characterize the sensitivity of SAC's performance with respect to the entropy scale
hyperparameter.

\end{document}

%% file: style/math_commands.tex
\usepackage{amsmath,amsfonts,bm}

\def\eqref#1{equation~\ref{#1}}

\def\1{\bm{1}}

\DeclareMathAlphabet{\mathsfit}{\encodingdefault}{\sfdefault}{m}{sl}
\SetMathAlphabet{\mathsfit}{bold}{\encodingdefault}{\sfdefault}{bx}{n}

\DeclareMathOperator*{\argmax}{arg\,max}
\DeclareMathOperator*{\argmin}{arg\,min}